\documentclass{article}
\usepackage{ijcai19}

\usepackage{times}
\usepackage{soul}
\usepackage{url}
\usepackage[utf8]{inputenc}
\usepackage[small]{caption}
\usepackage{graphicx}
\usepackage{amsmath}
\usepackage{booktabs}
\usepackage{algorithm}
\usepackage{algorithmic}
\usepackage{amsfonts}
\usepackage{amsthm}
\usepackage{amssymb}
\usepackage{microtype}
\graphicspath{{img/}{../img/}}
\usepackage{subfigure}
\usepackage{booktabs} 
\usepackage{epstopdf}
\usepackage{float}
\usepackage{subfiles}
\usepackage[english]{babel}
\usepackage{blindtext}
\usepackage{tabularx}
\usepackage{enumerate}
\usepackage{courier}  
\usepackage{graphicx}  
\usepackage{color}
\usepackage[colorlinks,linkcolor=blue,anchorcolor=blue,citecolor=blue,]{hyperref}

\newtheorem{theorem}{Theorem}
\newtheorem{lemma}{Lemma}
\newtheorem{remark}{Remark}
\newtheorem{corollary}{Corollary}

\makeatletter
\newcommand{\E}{\mathbb{E}}
\newcommand{\RB}{\mathbb{R}}
\newcommand{\KL}{\mathbb{KL}}
\newcommand{\M}{\mathcal{M}}
\newcommand{\W}{\mathcal{W}}
\newcommand{\XM}{\mathcal{X}}
\newcommand{\G}{\mathcal{G}}
\newcommand{\V}{\mathcal{V}}
\newcommand{\N}{\mathcal{N}}
\newcommand{\EM}{\mathcal{E}}
\newcommand{\pb}{\mathbf{p}}
\newcommand{\OM}{\mathcal{O}}

\makeatother

\title{An Asynchronous Decentralized Algorithm for \\Wasserstein Barycenter Problem}

\author{
Chao Zhang$^1$
\and
Hui Qian$^1$
\and
Jiahao Xie$^{1}$
\affiliations
$^1$Zhejiang University
\emails
\{zczju,qianhui,xiejh\}@zju.edu.cn,
}

\begin{document}

\maketitle

\begin{abstract}
	Wasserstein Barycenter Problem (WBP) has recently received much attention in the field of artificial intelligence. 
	In this paper, we focus on the decentralized setting for WBP and propose an asynchronous decentralized algorithm (A$^2$DWB). 
	A$^2$DWB is induced by a novel stochastic block coordinate descent method to optimize the dual of entropy regularized WBP. 
	To our knowledge, A$^2$DWB is the first asynchronous decentralized algorithm for WBP.
	Unlike its synchronous counterpart, it updates local variables in a manner that only relies on the stale neighbor information, which effectively alleviate the waiting overhead, and thus substantially improve the time efficiency.
	Empirical results validate its superior performance compared to the latest synchronous algorithm.
\end{abstract}

\section{Introduction}
In this paper, we consider the Wasserstein barycenter problem(WBP) in the semi-discrete setting, which estimates the barycenter of a set of continuous probability distributions under the Wasserstein distance, i.e.,  
\begin{equation}\label{wbo}
\min_{\nu} \sum_{i=1}^m \W_\beta(\mu_i,\nu),
\end{equation}
where $\W_{\beta}$ is the (regularized-)Wasserstein distance, $\mu_i$'s are a set of continuous distributions, and $\nu$ is a discrete distribution on a fixed support of size $n$. 
WBP has recently attracted much attention in the artificial intelligence literature due to its promising performance in many data analysis  and machine learning applications \cite{bigot2018characterization,cuturi2014fast,li2008real,ye2014scaling,ye2017fast,courty2017optimal,srivastavA2018scalable}.

Most of the researches on WBP focus on single node settings where all the information are stored in a single machine. \cite{benamou2015iterative,borgwardt2018improved,yang2018fast,cuturi2014fast,carlier2015numerical,claici2018stochastic}.
However, when data is distributed over a network of sensors\cite{nedic2017distributed,nedic2017fast}, or the transmission of information is limited by communication or privacy constraints\cite{dvurechenskii2018decentralize}, such methods will be infeasible.
This necessitates the development of distributed schemes for WBP.

Although, off-the-shelf distributed optimizers seem to be of the ready availability for WBP computation, their virtues are realized at the expense of a substantial computational burden, since calculating the gradient of the Wasserstein distance in the semi-discrete setting is itself a difficult stochastic optimization problem\cite{peyre2019computational}. The recent interests on this topic are mainly focused on exploring WBP's dual formulation\cite{staib2017parallel,dvurechenskii2018decentralize}. 
They share the same idea that an inducing method is adopted to solve the dual problem of WBP.
This primal-dual transformation results in a natural distributed form and in the meanwhile avoids the computational burden of calculating Wasserstain distance or its gradient.

Among these work, \citeauthor{staib2017parallel} proposed a \emph{centralized} distributed model for WBP, which adopts a stochastic projected sub-gradient descent as its inducing method \cite{staib2017parallel}. 
Usually, centralized schemes are not robust to machine failures or network topological changes, and have long synchronization time. 
\citeauthor{dvurechenskii2018decentralize} proposed a \emph{decentralized} algorithm with a primal-dual stochastic gradient scheme as its inducing method, which only requires local communication between neighboring nodes, in contrast to the centralized distributed model \cite{dvurechenskii2018decentralize}. 
However, this method still needs to \emph{synchronize} all the computational nodes in each iteration, and thus all nodes must wait for the slowest communication edge in the network.

Inspired by the above work, an \emph{asynchronous} decentralized algorithm for WBP is proposed in this paper.
Similar to \cite{dvurechenskii2018decentralize}, we also work on the dual problem and use entropic regularization to take advantage of the problem smoothness. Nevertheless, our algorithm enjoys the smaller waiting overhead without compromising the rates of convergence.
Note that existing literature on asynchronous decentralized optimization cannot be directly applied to WBP in the semi-discrete setting since almost all of them address the primal problem and calculating the gradient of the Wasserstein distance are required. \cite{zhang2018asyspa,hendrikx2018accelerated,hendrikx2019asynchronous,lan2018asynchronous,lian2017asynchronous}. 
Our contribution are listed as follows.
\begin{itemize}
	\item An asynchronous decentralized algorithm (A$^2$DWB) for WBP is proposed by applying a novel stochastic block coordinate descent scheme(ASBCDS) as its inducing method, which only relies on the stale information of its neighbour.
	To the best of our knowledge, this is the first asynchronous decentralized algorithm for WBP.
	\item 	We prove that ASBCDS achieves the optimal convergence rate for stochastic smooth optimization.
	When applying ASBCDS to the dual of a decentralized problem, the distance to the primal optimality and the consensus distance  converge in the same order as the dual objective.
	It therefore provides theoretical guarantees of A$^2$DWB.
\end{itemize}
We conduct a simulated experiment on calculating the Wasserstein barycenter of a set of Gaussian distributions and a real-world experiment on MNIST dataset.
Empirical results validate the time efficiency of A$^2$DWB compared to the synchronous algorithm.
\section{Notation and Preliminaries}
For a vector $x\in \RB^n$, we use $x^{[l]}$ to denote its $l$-th block or coordinate, which will be clear from the context.
We denote $I_{n \times n}$ as the $n \times n$ identity matrix.
$[A]_{ij}$ denotes the $j$-th element in the $i$-th row of a given matrix $A$.
Given a positive semi-definite matrix $W$, $\lambda_{max}(W)$ denotes its largest eigenvalue and $\sqrt{W}$ denotes its square root matrix.
$A \otimes B$ denotes the Kronecker's product of two given matrices $A$ and $B$, .
For a positive number $m$, perm($m$) denotes the random permutation of $\{1,\cdots,m\}$.
We denote $\M_+^1(\RB^n)$ as the set of all Radon measures on $\RB^n$ and $S_1(n)=\{a \in \RB^n_+| \sum_{i=1}^n a^{[i]} =1\}$ as the probability simplex.
Given a graph $\G =(\V,\EM)$ with nodes $\V$ and edges $\EM$, the Laplacian matrix $\bar W$ of $\G$ is defined as 
\begin{displaymath}
[\bar W]_{ij} =\left\{ \begin{array}{ll}
-1,& \textrm{if } (i,j) \in \EM,\\
\textrm{deg}(i), & \textrm{if }i=j,\\
0,& \textrm{otherwise},
\end{array}\right.
\end{displaymath}
where deg($i$) denotes the degree of node $i$.

\subsection{Decentralized Wasserstein Barycenter}\label{DWB}
Given a Radon measure $\mu \in \M^1_+(\RB^n)$ on $\RB^n$ and a discrete probability measure $\nu = \sum_{i=1}^n p^{[i]} \delta (z_i)$ with finite support $z_i$'s $\in \RB^n$ and weight $p\in S_1(n)$, the entropy regularized semi-discrete Wasserstein distance between the continuous measure $\mu$ and discrete measure $\nu$ is defined as
\begin{equation*}
\W_\beta(\mu,\nu) = \min_{\pi \in \Pi(\mu,\nu)}\{\sum_{i=1}^n \int c_i(y)\pi_i(y) dy + \beta \KL(\pi|\vartheta)\},
\end{equation*}
where $c_i(y) = c(z_i,y)$ is a cost function for transporting one unit mass from $z_i$ to $y$, $\vartheta$ is the uniform distribution on the support of $\pi$,
$\KL(\pi|\vartheta) = \sum_{i=1}^n \int \pi_i(y) \log(\frac{\pi_i(y)}{\vartheta})dy$ is the KL-Divergence between $\pi$ and $\vartheta$,
and 
\begin{align*}
\Pi(\mu,\nu) = \{\pi \in &\M_+^1(R^n)\times S_1(n)| 
\\
&\sum_{i=1}^n \pi_i(y) = \mu(y),\int \pi_i(y) dy = p_i\}
\end{align*}
denotes the admissible transportation from $\mu$ to $\nu$. 

The regularized semi-discrete Wasserstein Barycenter in the decentralized setting is defined as the solution of the following optimization problem
\begin{equation}\label{wb}
\min_{p_1=\cdots=p_m,\atop p_1,\cdots,p_m \in S_1(n)} \sum_{i=1}^m \W_\beta(\mu_i,\nu_i),
\end{equation}
where $\mu_i$ is stored on the $i$-th node and all the $\nu_i$'s are discrete probability measures with weight $p_i$ on fixed support $\{z_1,\cdots,z_n\}$.
We assume that each node communicates over a static, connected and undirected graph $\G = (\V,\EM)$.
The graph imposes information constraints, specifically, each node $i$ only has access to its local $\mu_i$ and two nodes $i$ and $j$ are allowed to exchange information with each other if they are neighbors, i.e., $(i,j) \in \EM$.

If we denote the Laplacian matrix of the graph $\G$ as $\bar W$ and write 
$\W_{\beta,\mu_i}(p_i) =\W_{\beta}(\mu_i,\nu_i)$, problem (\ref{wb}) is equivalent to the following problem
\begin{equation}\label{wbl}
\min_{\sqrt{W}\pb =0, \atop p_1,\cdots,p_m \in S_1(n)} \W(\pb) = \sum_{i=1}^m \W_{\beta,\mu_i}(p_i),
\end{equation}
where $\pb=[p_1^T,\cdots,p_m^T]^T$ denotes the concatenation of the local  variable $p_i$'s and $W = \bar W \otimes I_{n\times n}$. 
Here, each node aims to minimize the global objective with its local information,
while it also needs to ensure that its local $\nu_i$ is equal to that of its neighbors.
It can be shown that $\W(\pb)$ is $\beta$-strongly convex.

Directly calculating $W_{\beta,\mu_i}(p_i)$ or its gradient is quite involved and needs to solve another difficult optimization problem.
\cite{dvurechenskii2018decentralize} transformed the constrained primal problem (\ref{wbl}) to the unconstrained dual form
\begin{eqnarray}\label{wbd}
\min_{\bold{\eta} \in \RB^{mn}} \W^*(\bold{\eta}) = \sum_{i=1}^m \W_{\beta,\mu_i}^*([\sqrt{W}\bold{\eta}]^{[i]}),
\end{eqnarray}
where $\bold{\eta} = [\eta_1^T,\cdots,\eta_m^T]^T$ denotes the stacked vector of dual variables, $[\sqrt{W}\bold{\eta}]^{[i]}$ denotes the $i$-th block of $\sqrt{W}\bold{\eta}$, and 
$$
\W_{\beta,\mu_i}^*(\bar \eta) = \E_{y \sim \mu_i} \beta\log(\frac{1}{\mu_i(y)}\sum_{i=1}^n \exp(\frac{\bar\eta^{[i]} - c_i(y)}{\beta}))
$$
denotes the Fenchel-Legendre dual function \cite{rockafellar2015convex} of $\W_{\beta,\mu_i}$.
They proposed an accelerated primal-dual stochastic gradient method to solve the dual problem, 
where updating a block in $\eta$ is equivalent to updating the local variable of a node in the decentralized network.
In each iteration, their method needs a global synchronization and every node need to communicate with its neighbors.
Their method needs $\OM(1/\sqrt{\epsilon})$ iterations to get $\epsilon$-accuracy on the primal optimality and consensus distance.
The properties of $\W^*(\eta)$ are summarized in the following lemma.
\begin{lemma}[Lemma 2 in \cite{dvurechenskii2018decentralize}]\label{lemma:wb}
$\W^*(\bold{\eta})$ is $(\lambda_{max}(W)/\beta)$-smooth and its $i$-th  stochastic partial gradient approximation is defined as
\begin{align}
\tilde \nabla \W^*(\bold{\eta})^{[i]} = \sum_{j=1}^m [\sqrt{W}]_{ij} \tilde \nabla \W^*_{\beta,\mu_i} (\bar \eta_j), i = 1,\cdots,m.
\end{align}
Here, $\bar \eta_j = [\sqrt{W} \eta]^{[j]}$ and $\tilde \nabla \W^*_{\beta,\mu_i} (\bar \eta_j) = \frac{1}{M}\sum_{r=1}^M p_j^{(r)}(\bar \eta_j)$ is the average of $M$ samples of $p_j(\bar \eta_j)$ with the $l$-th dimension of $p_j(\bar \eta_j)$ defined as
\begin{align}
p_j(\bar \eta_j)^{[l]} = \frac{\exp((\bar \eta_j^{[l]} -c_l(Y_j))/\beta)}{\sum_{l=1}^n \exp((\bar \eta_j^{[l]} -c_l(Y_j))/\beta)},
\end{align}
where $Y_j$ is a sample from $\mu_j$.
We have $\E \tilde \nabla \W^*(\bold{\eta}) = \nabla \W^*(\bold{\eta})$ and $\E\| \tilde \nabla \W^*(\bold{\eta}) - \nabla \W^*(\bold{\eta})\|^2 \le \lambda_{max}(W)/M$.
\end{lemma}

\subsection{A General Primal-Dual Formulation}
In this subsection, we abstract the property of the primal problem (\ref{wbl}) and its dual (\ref{wbd}) and give a general primal-dual formulation of this problem.

The primal problem of a general decentralized optimization can be written as follows.
\begin{equation}\label{primal}
\min_{x \in \XM} F(x) \quad s.t. \sqrt{W} x = 0.  
\end{equation}
The dual problem of (\ref{primal}) is
\begin{equation}\label{dual}
\min \varphi(\eta),
\end{equation}
where $\varphi(\eta) = \max_x \langle \eta,\sqrt{W} x\rangle - F(x)$.
The gradient of $\varphi(\eta)$ is given by 
$
\nabla \varphi(\eta) = \sqrt{W} x^*(\sqrt{W}\eta),
$
where $x^*(\sqrt{W} \eta) = \max_x \langle \sqrt{W} \eta,x\rangle -F(x)$ is called the primal of $\eta$.

As stated in Lemma \ref{lemma:wb}, $F(x)$ and $\varphi(\eta)$ have the following properties.
\begin{itemize}
	\item $F(x)$ is $\mu$-strongly convex and $\varphi(\eta)$ is $L=\lambda_{max}(W)/\mu$ -smooth.
	\item $\varphi(\eta)$ is a stochastic function, i.e., $\varphi(\eta) = \E_{\xi} \phi(\eta,\xi)$ and we can get access to its stochastic gradient $\nabla \phi(\eta,\xi)$.
	\item $\E [\nabla \phi(\eta,\xi) ]=\nabla \varphi(\eta)$ and its variance is bounded by $\E \|\nabla \varphi(\eta) - \nabla \phi(\eta,\xi)\|^2 \le \sigma^2$.
\end{itemize}

According to the strong duality property, the duality gap is zero since $F(x)$ is strongly convex. 
In the following theorem, we establish the relation between the solution of the primal problem and the dual problem\footnote{Due to the limit of space, we put the proofs of all the theorems in the appendix, which is uploaded on an anonymous website \url{https://drop.me/BNKjWE}.}.
\begin{theorem}\label{them:dual}
Given a dual variable $\eta$ and its primal variable $x = x^*(\sqrt{W}\eta)$, the distance between $x$ and the optimum of $F(x)$, i.e. $x^*$, is bounded by $\|x - x^*\|^2 \le \frac{2}{\mu}(\varphi(\eta) - \varphi(\eta^*))$ and the consensus distance is bounded by
$\|\sqrt{W} x\|^2 \le \frac{\lambda_{max}(W)}{\mu}(\varphi(\eta)-\varphi(\eta^*))$,
where $\eta^*$ is the optimum of $\varphi(\eta)$.
\end{theorem}

Theorem \ref{them:dual} indicates that the primal problem can be solved effectively by computing a solution of its dual by an inducing method.
As an update in one block of the dual variable corresponds to an update on the local variable of a node,
we resort to the stochastic block coordinate descent framework to obtain an algorithm where each node can update its own variable separately.
Besides, in order to be asynchronous and match the iteration complexity of the synchronous algorithm,
the inducing method should allow to use stale information and achieve acceleration.
Existing stochastic block coordinate descent algorithms can not be directly used here, since they cannot satisfy all the requirements simultaneously.

\section{Algorithms and Convergence Analyses}
In this section, we present the main algorithms and their convergence analyses.
We first introduce a novel accelerated stochastic block coordinate descent algorithm, which can use the stale information and achieve $\OM(1/\sqrt{\epsilon})$ iteration complexity.
Then, we propose a practical implementation PASBCDS of ASBCDS, and show the equivalence of these two algorithms.
Finally, a practical asynchronous accelerated  decentralized algorithm (A$^2$DWB) for WBP was proposed by utilizing PASBCDS as the inducing method.

\subsection{Accelerated Stochastic Block Coordinate Descent Algorithm with Stale Information}
Inspired by \cite{fang2018accelerating}, we propose an Accelerated Stochastic Block Coordinate Descent method with Stale information (ASBCDS) and list its details in Algorhtm \ref{alg:ASBCD}. 
We follow the 3 sequence acceleration technique used in \cite{fercoq2015accelerated} to update the variable $\lambda$, $\zeta$ and $\eta$.
Out-of-date variables are allowed to be used to calculate the descent direction.
However, instead of directly using the gradient of stale $\eta_{j_p(k+1)}^{[p]}$,
we first compensate it with $\sum_{i = j_p(k+1)}^k\rho_i (\lambda_{j_p(k+1)}^{[p]} - \eta_{j_p(k+1)-1}^{[p]})$ and then use the compensated variable $\omega_{j(k+1)}$to calculate the gradient.
As indicated by \cite{fang2018accelerating}, this compensation is critical to get an acceleration in this setting.
\begin{algorithm}[t]
	\caption{Accelerated  Stochastic Block Coordinate Descent Algorithm with Stale Information(ASBCDS)}
	\label{alg:ASBCD}
	\textbf{Input}: Initialization $\eta_0=\zeta_0=\lambda_0$, number of iteration $N$, $\theta_1 = 1/n$ and learning rate $\gamma$.
	\begin{algorithmic}[1] 
		\FOR{$k = 0,1,\cdots,K$}
		\STATE Set $\lambda_{k+1} = \theta_{k+1} \zeta_k + (1-\theta_{k+1})\eta_k$.
		\STATE For $p \in \{1,\cdots,m\}$, calculate $d_k =\frac{\theta_{k+1}(1-\theta_k)}{\theta_k}$, update $\rho_i = \prod_{l=j_p(k+1)}^i d_l$ and calculate
		\begin{small}
			\begin{align*}
			\omega_{j(k+1)}^{[p]} = &\eta_{j_p(k+1)}^{[p]}
			+\sum_{i = j_p(k+1)}^k\rho_i (\lambda_{j_p(k+1)}^{[p]} - \eta_{j_p(k+1)-1}^{[p]}). 
			\end{align*}
		\end{small}
		\STATE Choose $i_{k}$ and get stochastic partial gradient $g_{k+1} = \nabla \phi(\omega_{j(k+1)},\xi_{k+1})^{[i_{k}]}$, then calculate 
		\begin{displaymath}
		\zeta_{k+1} ^{[i_{k}]}=\zeta_{k}^{[i_{k}]} -\frac{\gamma}{m\theta_{k+1}}g_{k+1},
		\end{displaymath}
		and $\zeta_{k+1} ^{[i]} = \zeta_{k}^{[i]}$ for $i \ne i_k$.
		\STATE Calculate $\eta_{k+1} = \lambda_{k+1} + m\theta_{k+1}(\zeta_{k+1}-\zeta_k)$.
		\STATE Update  $\theta_{k+2} = \frac{\sqrt{\theta_{k+1}^4 +4 \theta_{k+1}^2 }-\theta_{k+1}^2}{2}$.
		\ENDFOR
	\end{algorithmic}
	\textbf{Output}: $\eta_{K+1}$.
\end{algorithm}

In the following lemma, we summarize the properties of the sequence $\theta_k$ in Algorithm \ref{alg:ASBCD}.
\begin{lemma}\label{lemma:theta}
Assume that $\theta_1 =  1/m$ and $\theta_{k+1} = \frac{\sqrt{\theta^4_k+4 \theta_k^2}-\theta_k^2}{2}$ for $k \ge 1$, then $\theta_k$ satisfies $\frac{1}{k-1+2m}\le \theta_k \le \frac{2}{k-1+2m}$ and $\frac{1-\theta_{k+1}}{\theta_{k+1}^2} = \frac{1}{\theta_k^2}$.
\end{lemma}
We now give the main convergence results of ASBCDS.
\begin{theorem}\label{them:main}
Assume that $\varphi(\eta)$ is $L$-smooth, the variance of stochastic gradient is bounded as $\E \| \nabla \varphi(\omega_{j(k)})- \nabla \phi(\omega_{j(k)},\xi_{k})\|^2  \le \frac{mL\theta_{k}\epsilon}{8}$, and the delay $\tau \le m$.
Then for Algorithm \ref{alg:ASBCD}, if the learning rate $\gamma$ satisfies $3 L \gamma + 12 L \gamma(\frac{\tau^2+\tau }{m}+2\tau)^2 \le 1$, we have $\E \varphi(\eta_k) -\varphi(\eta^*) \le \epsilon$ after $K =\OM(\frac{m\tau\sqrt{L}}{\sqrt{\epsilon}})$ iteration.
Moreover, if we replace the bounded variance assumption on $\nabla \phi(\omega_{j(k)},\xi_{k})$ with $\E \| \nabla \varphi(\lambda)- \nabla \phi(\lambda,\xi)\|^2  \le\sigma^2$ and sample $M_k$ mini-batch of samples in the $k$-th iteration, the total number of stochastic partial gradient oracle access is bounded by $\OM(\frac{m\tau \sqrt{L}}{\sqrt{\epsilon}}+\frac{m\tau^2\sigma^2}{\epsilon^2})$.
\end{theorem}
\begin{proof}[Sketch of Proof]
	We give the sketch of proof here.
	For the full proof, please check the appendix.
	We prove our results in five step.
	First we bound $\|\lambda_{k+1} -\omega_{j(k+1)}\|^2$, $\varphi(\eta_{k+1}) -\varphi(\eta^*)$ and $\|\zeta_{k+1} -\eta^*\|^2$ separately, and then combine these results together and choose proper parameters to obtain the convergence rate.	
\begin{itemize}
	\item [Step 1:]By expanding $\lambda_{k+1}$ through the update rule of ASBCDS,
	we can establish that the norm of $\lambda_{k+1} -\omega_{j(k+1)}$ is bounded by
	\begin{align*}
	\resizebox{1\linewidth}{!}{$
	\sum_{p=1}^{m} (\frac{\tau^2+\tau}{2m} + \tau) \sum_{i=1}^{\min(k-\tau,\tau)}(1+\frac{i}{m})
	\|\eta_{k-i+1}^{[p]}-\lambda_{k-i+1}^{[p]}\|^2$}.
	\end{align*}
	\item[Step 2:]By analyzing the function value, we have 
	\begin{align}
	&\E_{i_k}\varphi(\eta_{k+1})\le\varphi(\lambda_{k+1})+  \frac{\gamma L^2}{2D_1}\|\lambda_{k+1}-\omega_{j(k+1)}\|^2 
	\nonumber\\
	&+ \frac{\gamma }{2mD_2}\| \nabla \varphi(\omega_{j(k+1)})- \nabla \phi(\omega_{j(k+1)},\xi_{k+1})\|^2  
	\nonumber\\
	&-\gamma(1-\frac{L\gamma}{2}-\frac{D_1+D_2}{2})\E_{i_k}\|\frac{\eta_{k+1}^{[i_k]}-\lambda_{k+1}^{[i_k]}}{\gamma}\|^2,
	\nonumber
	\end{align}
	where $D_1$ and $D_2$ are auxiliary constants.
	\item[Step 3:]We then prove the upper bound of $\|\zeta - \eta^*\|^2$ as
	\begin{align}
	&\quad \frac{m^2}{2\gamma}\E\|\theta_{k+1}\zeta_{k+1} - \theta_{k+1}\eta^*\|^2 
	\nonumber\\
	&= \frac{1}{2\gamma}\E\|\eta_{k+1}- \lambda_{k+1}\|^2 +\frac{m^2}{2\gamma}\|\theta_{k+1}\zeta_k -\theta_{k+1}\eta^* \|^2 
	\nonumber\\
	&\quad +(1-\theta_{k+1})\varphi(\eta_{k}) + \theta_{k+1} \varphi(\eta^*)- \varphi(\lambda_{k+1})
	\nonumber\\
	&\quad +\langle \nabla \varphi(\lambda_{k+1}) -\nabla \varphi(\omega_{j(k+1)} ),\lambda_{k+1} -  \omega_{jl(k+1)}).
	\nonumber
	\end{align}
	\item[Step 4:]By combing the results from step $1\sim 3$, and adding up from $k=0 $ to $K$, we have
	\begin{align}
	&\frac{\E \varphi(\eta_{K+1}) -\varphi(\eta^*)}{\theta_{K+1}^2} +\frac{m^2}{2\gamma}\E\|\zeta_{K+1} - \eta^*\|^2
	\nonumber\\
	\le &\frac{1}{\theta_1^2} (\varphi(\eta_0) - \varphi(\eta^*)) +\frac{m^2}{2\gamma}\|\zeta_0 -\eta^* \|^2 + \sum_{k=1}^K\frac{\epsilon }{16\theta_{k+1}}
	\nonumber\\
	&-\big[\frac{1}{2}- 4(\frac{\gamma^2 L^2}{2D_1}+L\gamma)
	(\frac{\tau^2+\tau}{m} + 2\tau)^2)
	\nonumber\\
	&-\frac{L\gamma+D_1+D_2}{2} \big]\sum_{k=0}^K\frac{\gamma}{\theta_{k+1}^2}\E\|\frac{\eta_{k+1}-\lambda_{k+1}}{\gamma}\|^2.
	\nonumber
	\end{align}
	\item[Step 5:]
	By figuring out the order of $\theta_k$'s and  choosing proper $D_1$, $D_2$, $\gamma$ and $M_k$, we can conclude the theorem.
	
\end{itemize}
\end{proof}

\begin{algorithm}[t]
	\caption{Practical Implementation of ASBCDS (PASBCDS)}
	\label{alg:PASBCD}
	\textbf{Input}: Initialization $\eta_0=\zeta_0=\lambda_0$, number of iteration $K$, $\theta_1 = 1/n$ and learning rate $\gamma$.
	\begin{algorithmic}[1] 
		\FOR{$k = 0,1,\cdots,K$}
		\STATE For $p \in \{1,\cdots,m\}$, calculate \begin{displaymath}
		\omega_{j(k+1)}^{[p]} = u_{j_p(k+1)}^{[p]} +\theta_{k+1}^2v_{j_p(k+1)}^{[p]} ,
		\end{displaymath}
		where $p = \{1,\cdots,m\}$.
		\STATE Choose $i_k$ and get stochastic partial gradient $g_{k+1} = \nabla \phi(\omega_{j(k+1)},\xi_{k+1})^{[i_k]}$, then calculate $\delta_{k+1} = \frac{\gamma}{m\theta_{k+1}}g_{k+1}$
		\STATE Calculate \begin{displaymath}
		u _{k+1}^{[i_k]} = u_k^{[i_k]}  - \delta_{k+1},v_{k+1}^{[i_k]}  = v_k^{[i_k]}  + \frac{1-m \theta_{k+1}}{\theta_{k+1}^2}\delta_{k+1},
		\end{displaymath}
		and $u _{k+1}^{[j]} = u_k^{[j]} $, $v_{k+1}^{[j]} = v_k^{[j]}$ for $j \ne i_k$. 
		\STATE Update  $\theta_{k+2} = \frac{\sqrt{\theta_{k+1}^4 +4 \theta_{k+1}^2 }-\theta_{k+1}^2}{2}$.
		\ENDFOR
	\end{algorithmic}
	\textbf{Output}: $\eta_{K+1}= u_{K+1} + \theta_{K+1}^2 v_{K+1}$.
\end{algorithm}

\begin{remark}
	Compared to the optimal stochastic partial gradient complexity $\OM(\frac{m \sqrt{L}}{\sqrt{\epsilon}}+\frac{m\sigma^2}{\epsilon^2})$ for the stochastic smooth optimization \cite{lan2012optimal}, the complexity of ASBCDS $\OM(\frac{m\tau \sqrt{L}}{\sqrt{\epsilon}}+\frac{m\tau^2\sigma^2}{\epsilon^2})$ matches it up to a constant factor $\tau$.
	There have been researches on establishing optimal stochastic block coordinate descent algorithms that use stale information\cite{hannah2018texttt,fang2018accelerating}.
	However, most of them focus on the deterministic settings and algorithms for the stochastic optimization has been less investigated.
	To the best of our knowledge, this is the first stochastic block coordinate descent algorithm with stale information that can achieve the optimal complexity for the stochastic smooth optimization.
\end{remark}

According to Theorem \ref{them:dual} and \ref{them:main},
we can establish the following corollary on the property of the primal variable $x^*(\eta_{K+1})$.

\begin{corollary}\label{them:consensus}
	When applying ASBCDS to the dual problem (\ref{dual}), after $K  = \OM({m \tau \sqrt{L}}/{\sqrt{\epsilon}})$ iterations,
	the distance between $x_{K+1} = x^*(\sqrt{W} \eta_{K+1})$ and the optimum of the primal objective is bounded by $\E\|x_{K+1} - x^*\|^2 = \OM({\epsilon}/{\mu})$ and the consensus distance is bounded by
	$\E\|\sqrt{W} x_{K+1}\|^2 =\OM( {\lambda_{max}(W)\epsilon }/{\mu})$,
	where $x^*$ and $\eta^*$ denote the optimal solution to the primal problem and the dual problem, respectively.
\end{corollary}

\subsection{Practical Implementation of ASBCDS}
Although the update of $\zeta$ is block-wise, we still need full vector operation in ASBCDS when updating $\lambda$ and $\eta$.
Besides, in line 3 of Algorithm \ref{alg:ASBCD}, we need to calculate $\rho_i$,
which is a little complicated.
To tackle these problems,
we follow similar change of variable technique used in \cite{fercoq2015accelerated,fang2018accelerating}, and rewrite ASBCDS into a new form.
This practical implementation is referred to as PASBCDS and detailed in Algorithm \ref{alg:PASBCD}.
The equivalence between these two algorithms is summarized as follows.
\begin{theorem}[Equivalence between ASBCDS and PASBCD]
	If we take the same $j_p(k+1)$ and $\xi_{k+1}$ in each iteration of Algorithm \ref{alg:ASBCD} and Algorithm \ref{alg:PASBCD}, then we have $\lambda_{k+1} =u_k + \theta_{k+1}^2 v_k $, $\zeta_{k+1} = u_{k+1}$ and $\eta_{k+1} = u_{k+1} + \theta_{k+1}^2 v_{k+1}$ for all $k = 0,\cdots, K$. 
\end{theorem}

\subsection{Asynchronous Accelerated Decentralized Wasserstein Barycenter algorithm}
In this subsection, we present a practical asynchronous algorithm A$^2$DWB for WBP, which is induced by applying PASBCDS to problem (\ref{wbd}).
The detail of A$^2$DWB is listed in Algorithm \ref{alg:ADWB}.
Similar to Algorithm 3 in \cite{dvurechenskii2018decentralize}, we change the variable and denote $\bar{\omega} =\sqrt{W} \omega$, $\bar u = \sqrt{W} u$ and $\bar v = \sqrt{W} v$.
At time $t_k$, one node $i_k$ is activated and its parameters are updated.
Note that instead of calculating $\omega_{j(k+1)}^{[p]} = u_{j_p(k+1)}^{[p]} +\theta_{k+1}^2v_{j_p(k+1)}^{[p]}$ and its stochastic partial gradient, 
we directly use the local stored information.
This is equivalent to compensate $u_{j_p(k+1)}^{[p]}$ with $\theta_{j_p(k+1)+1}^2v_{j_p(k+1)}^{[p]}$ for $p \ne i_k$.
Empirically, this works well since if the delay $\tau$ is not large, $\theta_{j_p(k+1)+1}^2\approx\theta_{k+1}^2$.

Practically, we need to specify the activation time $t_k$ and node $i_k$ in A$^2$DWB.
This can be implemented effectively in the following way:
a seed is distributed to each node at the beginning and then a sequence of $t_k$'s and $i_k$'s is generated with the common seed.
Each nodes then check the sequence to determine when it should be activated.
Note that the activation scheme is  determined based on a trade-off between speed and accuracy:
If the nodes is activated more frequently, then more iterations can be performed in a given time, 
but the local stale gradient will be more out-of-date which will deteriorate the accuracy.
On the other hand, if the activation interval is long,  each node can get more recent gradient from its neighbors at the cost of less iterations run in the same time period.

\begin{algorithm}[t]
	\caption{Asynchronous Accelerated  Decentralized Wasserstein Barycenter algorithm(A$^2$DWB)}
	\label{alg:ADWB}
	\textbf{Input}: Initialization $\bar u_0=\bar v_0=\bar\omega_0=0$, number of iteration $K$, $\theta_1 = 1/n$ and learning rate $\gamma$.
	\begin{algorithmic}[1] 
		\STATE Calculate $\tilde \nabla \W_{\gamma,\mu_i}([\bar\lambda_0]^{[i]})$ on each node $i$ and share it with its neighbors.
		\FOR{$k = 0,1,\cdots,K$}
		\STATE At time $t_k$, randomly choose a node $i_k$ and activate it.
		\IF{Node $i$ is activated}
		\STATE Calculate 
		\begin{displaymath}
		\bar \omega_{j(k+1)}^{[i]} = \bar u_{k}^{[i]} +\theta_{k+1}^2 \bar v_{k}^{[i]} .
		\end{displaymath}
		\STATE  Generate $M_{k+1}$ samples from measure $\mu_{i_k}$, set $g_{i} = \tilde \nabla \W^*_{\beta,\mu_{i}} (\bar \omega_{j(k+1)}^{[i]})$ as in Lemma \ref{lemma:wb} and broadcast it to its neighbors. 
		\STATE Update 
		$$\delta_{k+1} = \frac{\gamma}{m\theta_{k+1}}(g_{i}+\sum_{j\in\textrm{neigh}(i)} W_{ij} [\tilde \nabla \W^*_{\beta,\mu_j}]_{local} ),$$
		where $[\tilde \nabla \W^*_{\beta,\mu_j}]_{local} $ is the local stored stale gradient obtained from its neighbor $j$.
		\STATE Calculate \begin{displaymath}
		u _{k+1}^{[i]} = u_k^{[i]}  - \delta_{k+1},v_{k+1}^{[i]}  = v_k^{[i]}  + \frac{1-m \theta_{k+1}}{\theta_{k+1}^2}\delta_{k+1}.
		\end{displaymath}
		\ELSE 
		\STATE $u _{k+1}^{[j]} = u_k^{[j]} $, $v_{k+1}^{[j]} = v_k^{[j]}$ for $j \ne i_k$.
		\ENDIF
		\STATE Update  $\theta_{k+2} = \frac{\sqrt{\theta_{k+1}^4 +4 \theta_{k+1}^2 }-\theta_{k+1}^2}{2}$.
		\ENDFOR
	\end{algorithmic}
	\textbf{Output}: $\eta_{K+1}= u_{K+1} + \theta_{K+1}^2 v_{K+1}$.
\end{algorithm}
	\vskip -0.2in

\section{Experiments}
In this section, we present the experimental results for A$^2$DWB.
We compare it with its synchronous counterpart, i.e., Algorithm 3 (referred as DCWB) in \cite{dvurechenskii2018decentralize}.
To show the impact of compensation, 
we also include a naive asynchronous algorithm, named as A$^2$DWBN,  
where each node directly uses the stale gradient of $\eta_{j_p(k+1)}$ to update it local variable.
Following similar experiment setting of \cite{dvurechenskii2018decentralize}, 
we conduct empirical studies on two tasks:
one is a simulated experiment of calculating the Wasserstein barycenter of a set of Gaussian distributions and the other is a real world application, which computes the Wasserstein barycenter of samples of digit from the MNIST dataset.

In both of the tasks, we simulate a network with $m =500$ nodes.
The network topologies considered in our experiments, in descending order of connectivity, are complete, Erd\H{o}s-R\'{e}nyi, cycle and star graphs.
The communication time $t$ of one node transferring its information to its neighbor is generated from a categorical distribution with support $[0.2\textrm{s},0.4\textrm{s},0.6\textrm{s},0.8\textrm{s},1\textrm{s}]$, and $t$ is equally distributed on the support.
In A$^2$DWB and A$^2$DWBN, we activate all the nodes one by one according to perm$(m)$ once every $0.2$ second, i.e., the smallest time interval one node can receive the information from its neighbors.
We run both algorithms for 200 seconds and report the dual objective value and the consensus distance as the metrics of performance, since the distance to the primal optimum is hard to directly calculated and it is bounded by the dual optimality.  
\begin{figure}[t]
	\centering
	\begin{subfigure}
		\centering
		\includegraphics[trim={0cm 0cm 0cm 0cm},clip,width=3.7cm,height= 3.2cm]{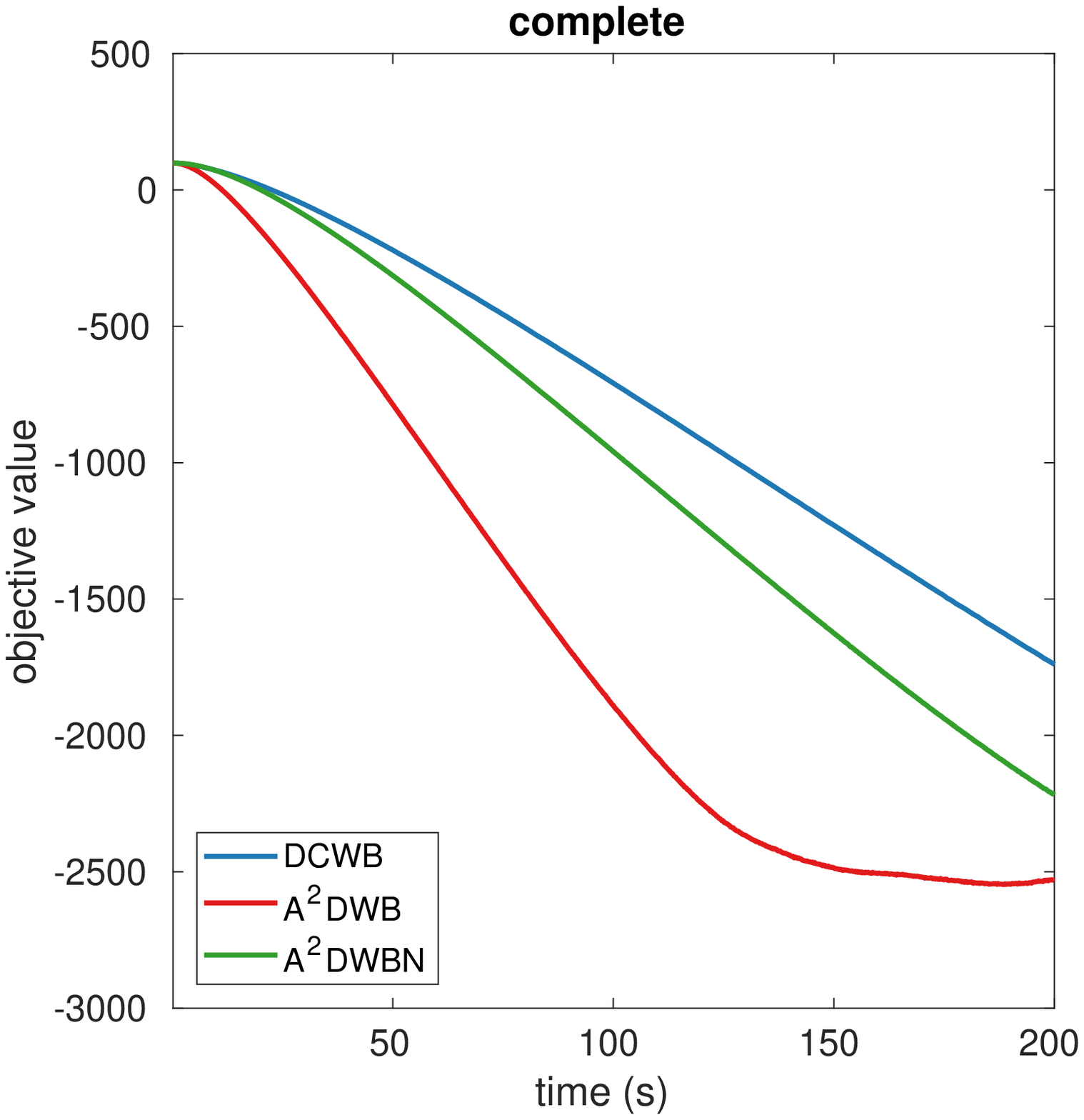}
	\end{subfigure}
	~
	\begin{subfigure}
		\centering
		\includegraphics[trim={0cm 0cm 0cm 0cm},clip,width=3.7cm,height= 3.2cm]{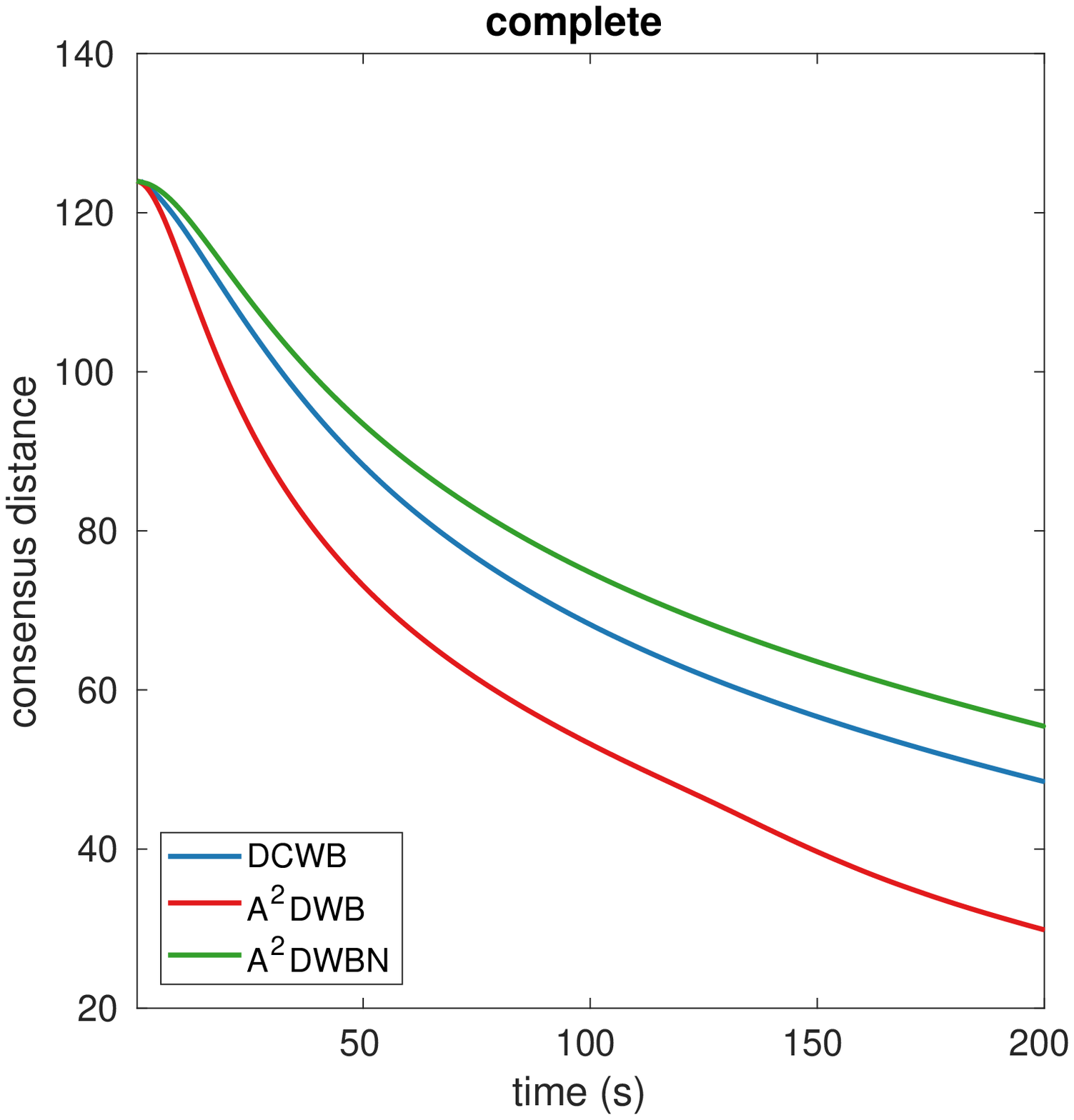}
	\end{subfigure}
	~
	\begin{subfigure}
		\centering
		\includegraphics[trim={0cm 0cm 0cm 0cm},clip,width=3.7cm,height= 3.2cm]{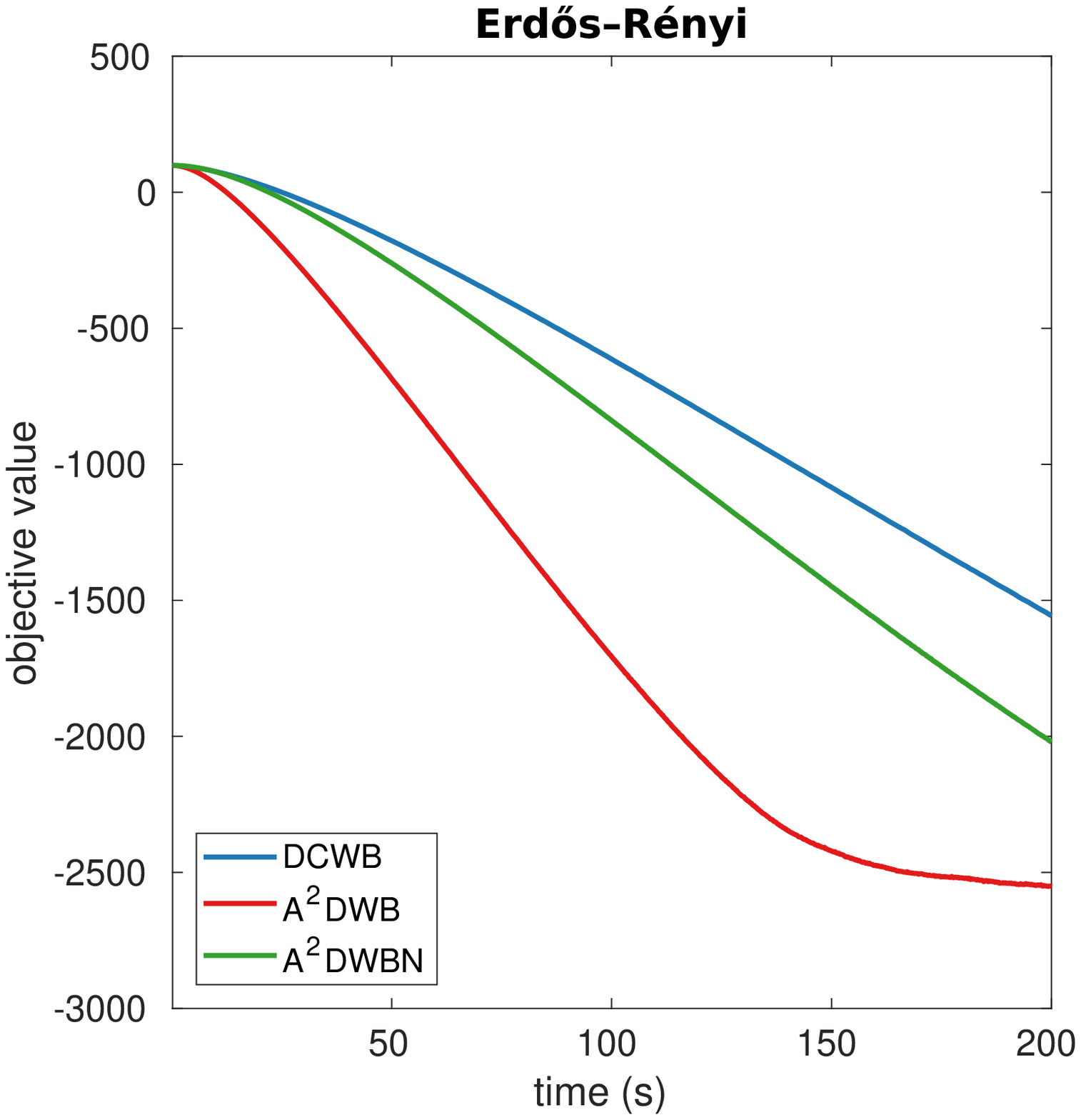}
	\end{subfigure}
	~
	\begin{subfigure}
		\centering
		\includegraphics[trim={0cm 0cm 0cm 0cm},clip,width=3.7cm,height= 3.2cm]{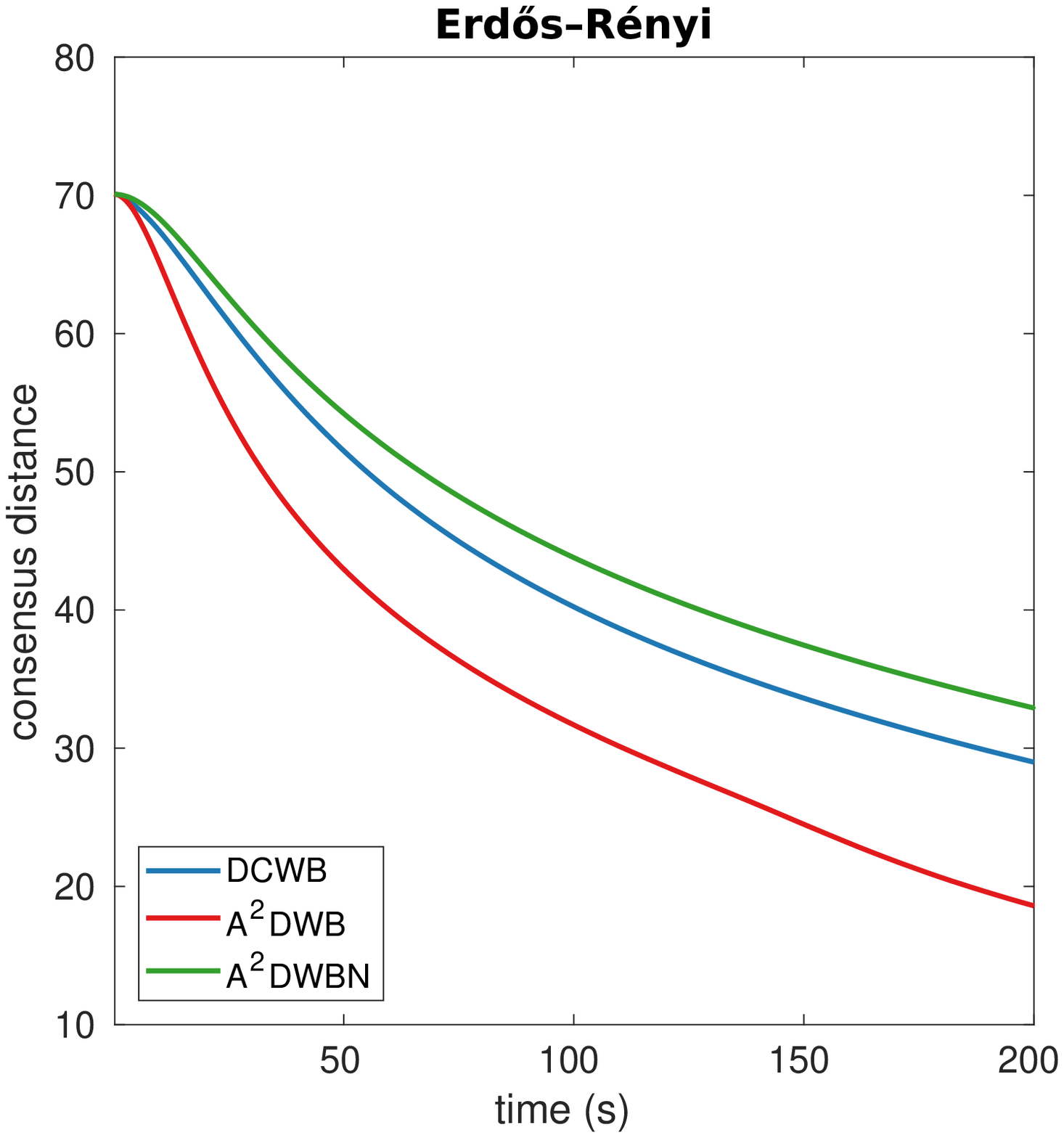}
	\end{subfigure}
	~
	\begin{subfigure}
		\centering
		\includegraphics[trim={0cm 0cm 0cm 0cm},clip,width=3.7cm,height= 3.2cm]{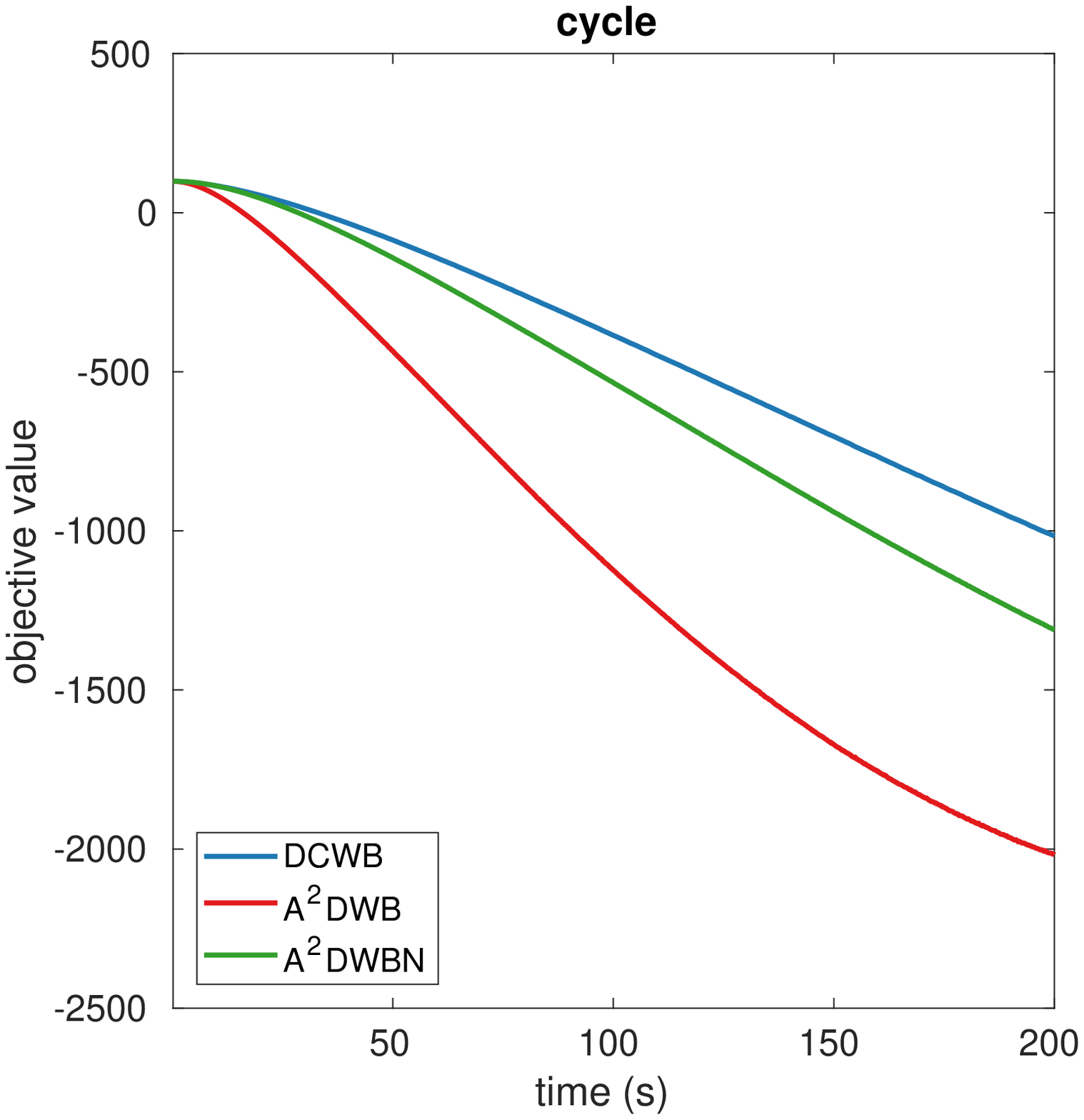}
	\end{subfigure}
	~
	\begin{subfigure}
		\centering
		\includegraphics[trim={0cm 0cm 0cm 0cm},clip,width=3.7cm,height= 3.2cm]{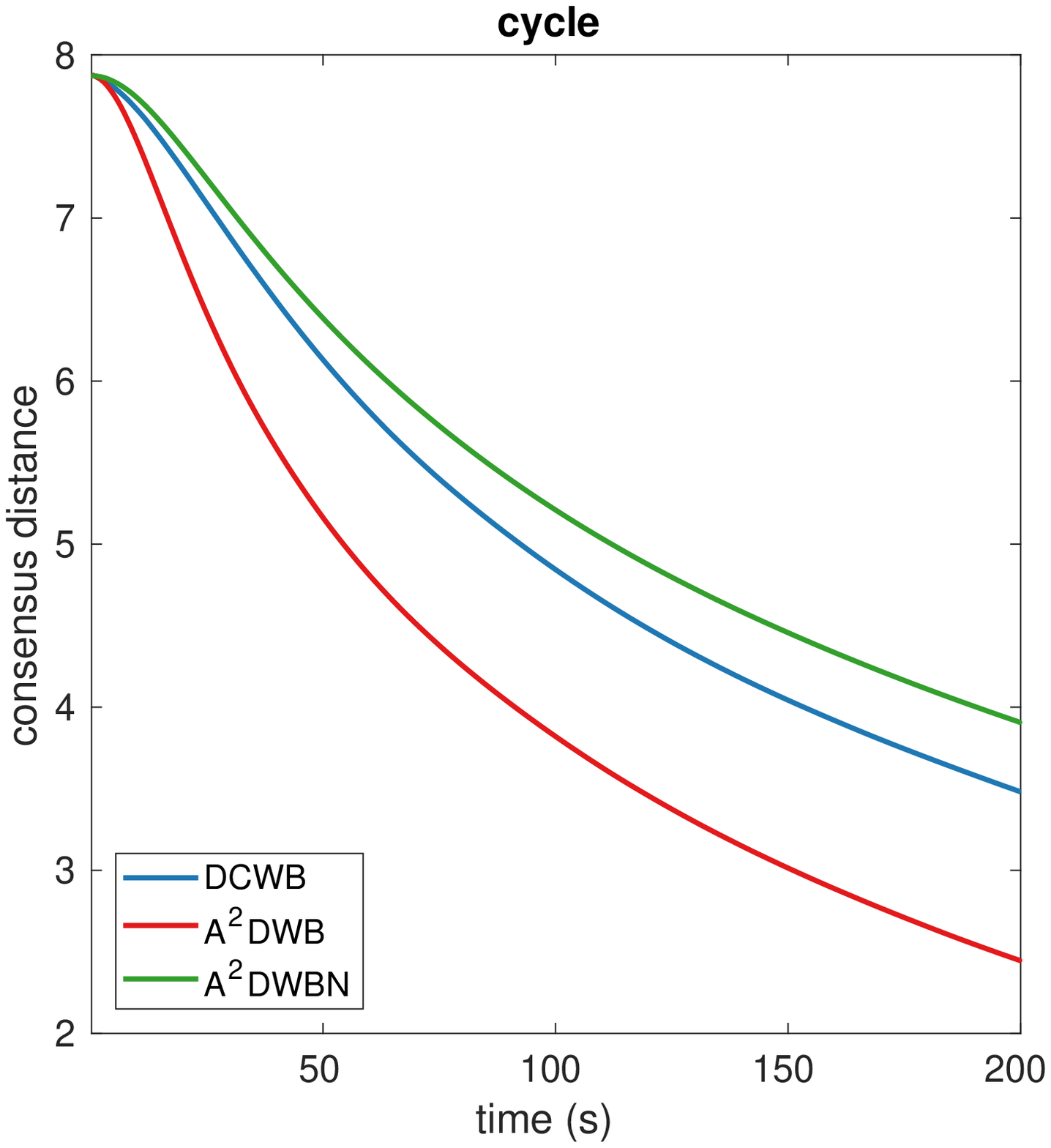}
	\end{subfigure}
	~
	\begin{subfigure}
		\centering
		\includegraphics[trim={0cm 0cm 0cm 0cm},clip,width=3.7cm,height= 3.2cm]{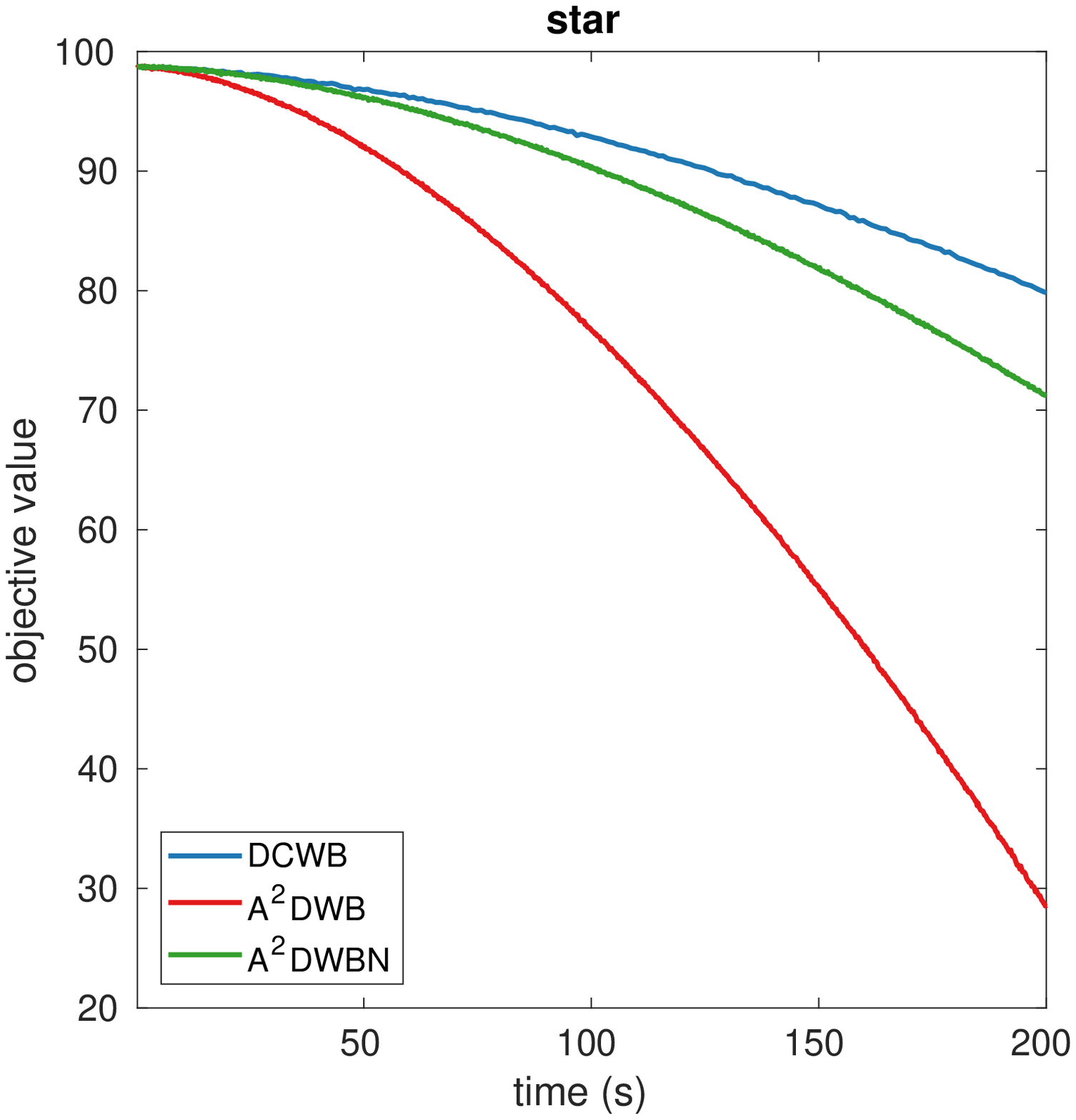}
	\end{subfigure}
	~
	\begin{subfigure}
		\centering
		\includegraphics[trim={0cm 0cm 0cm 0cm},clip,width=3.7cm,height= 3.2cm]{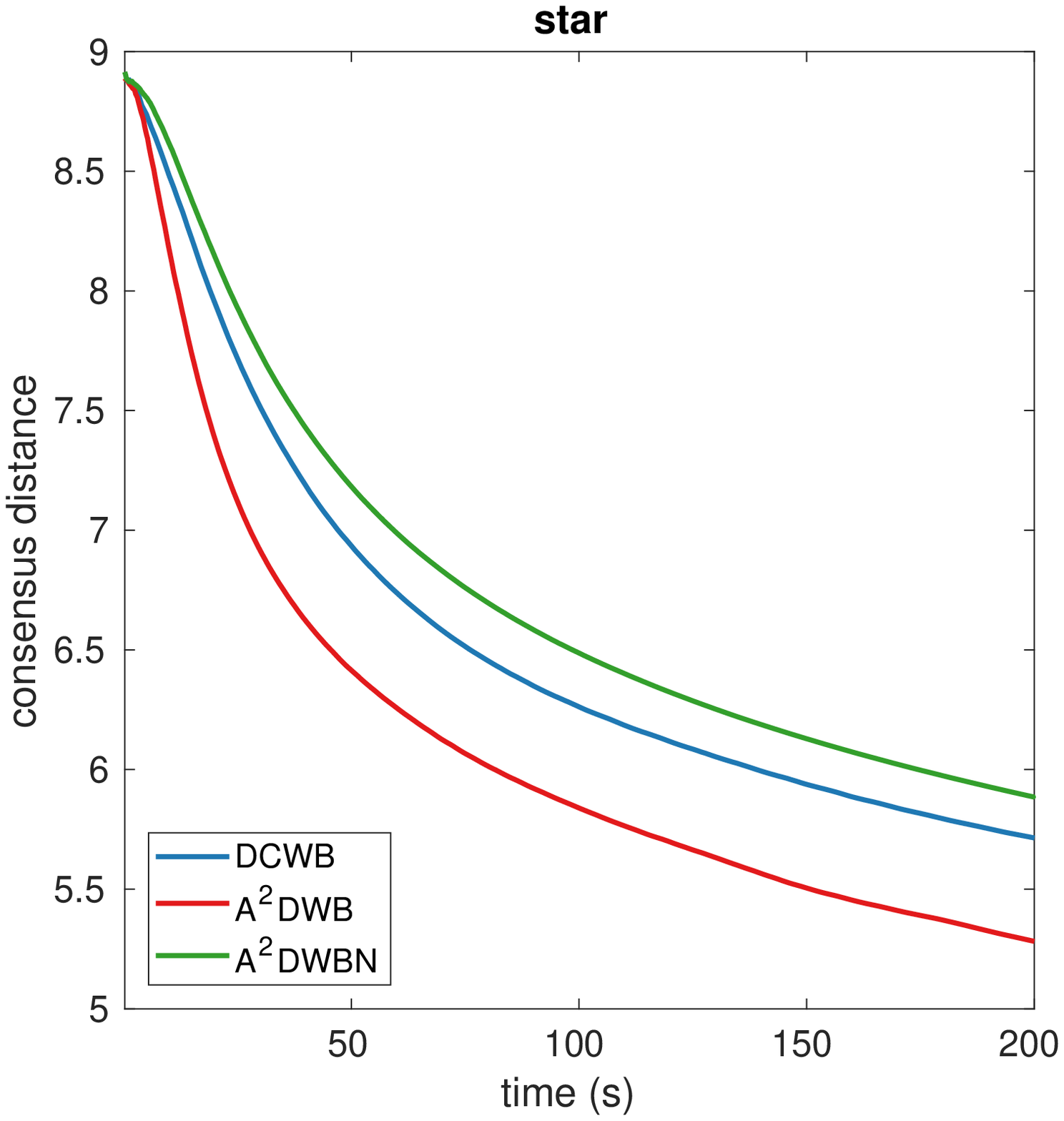}
	\end{subfigure}
	\caption{Simulation Experiment with Gaussian Distributions.
			From the top to the bottom, we sequentially report the results on complete, Erd\H{o}s-R\'{e}nyi, cycle and star networks.}
	\label{fig:exp1}
	\vskip -0.2in
\end{figure}

\subsection{Barycenter of Gaussian Distributions}
In the simulated experiment, we consider the following setting:
each node $i$ can query realizations from a privately held random distribution $\mu_i = \N(\theta_i,\sigma_i^2)$, where $\N(\theta_i,\sigma_i^2)$ is a univariate Gaussian distribution with mean $\theta_i$ and variance $\sigma_i^2$.
Each $\theta_i$  and $\sigma_i$ are randomly chosen from $[-4,4]$ and $[0.1,0.6]$, respectively.
The goal is to compute a discrete distribution $p \in S_1(n)$ that solves the WBP problem (\ref{wb}).
Here, we assume that $n=100$ and the support of $p$ consists of 100 points which are equally spaced on the segment $[-5,5]$.

The results are shown in Figure \ref{fig:exp1}.
From the figure, it can be verified that A$^2$DWB constantly outperforms the other algorithms in both the objective value and the consensus distance. 
Besides, we can see from the figure that the network topology do effect the convergence of the algorithms:
the convergence speeds drop severely as the connectivity of the networks reduces from complete-connected to poorly-connected.
\begin{figure}[!t]
	\centering
	\begin{subfigure}
		\centering
		\includegraphics[trim={0cm 0cm 0cm 0cm},clip,width=3.7cm,height= 3.2cm]{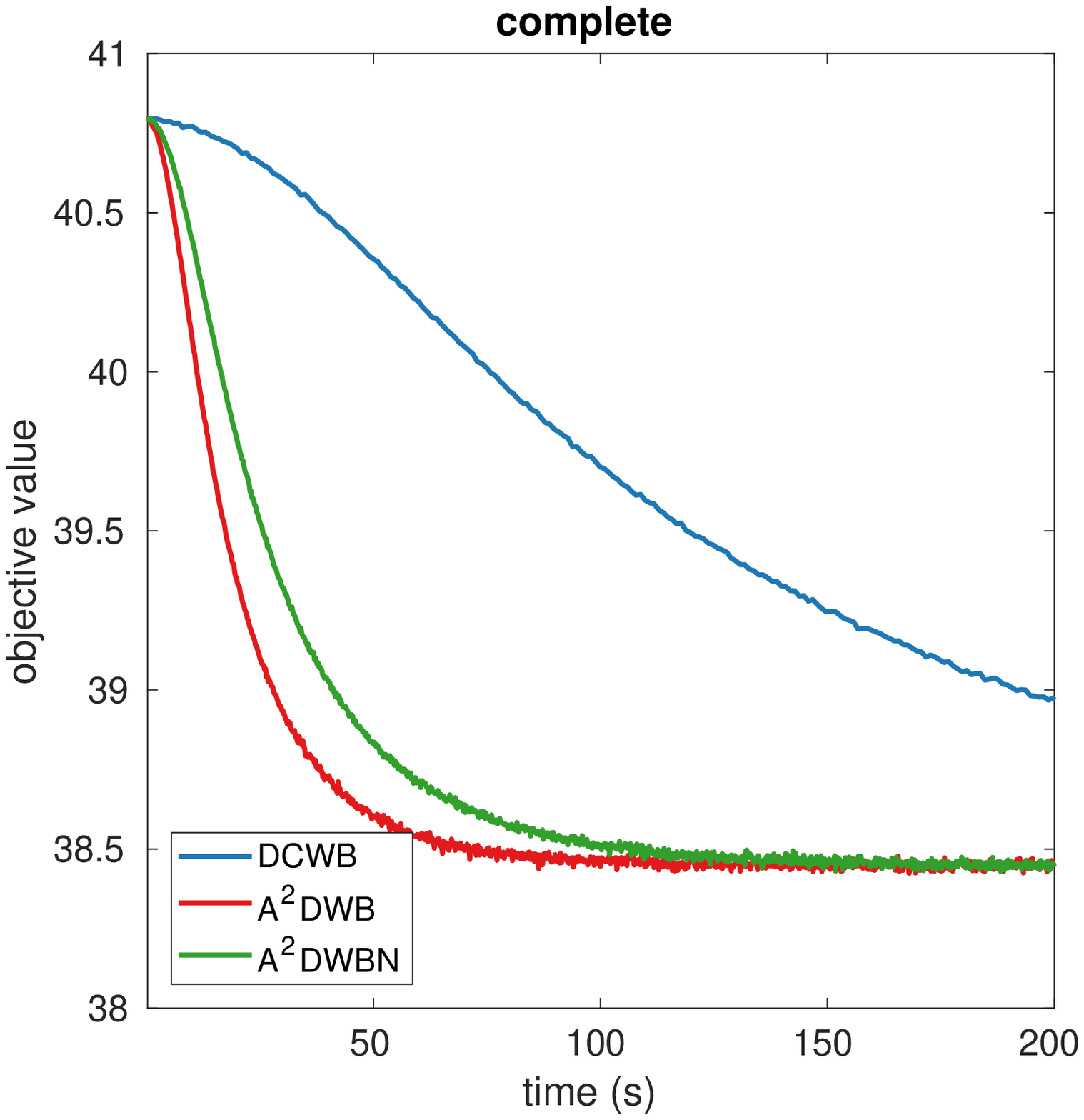}
	\end{subfigure}
	~
	\begin{subfigure}
		\centering
		\includegraphics[trim={0cm 0cm 0cm 0cm},clip,width=3.7cm,height= 3.2cm]{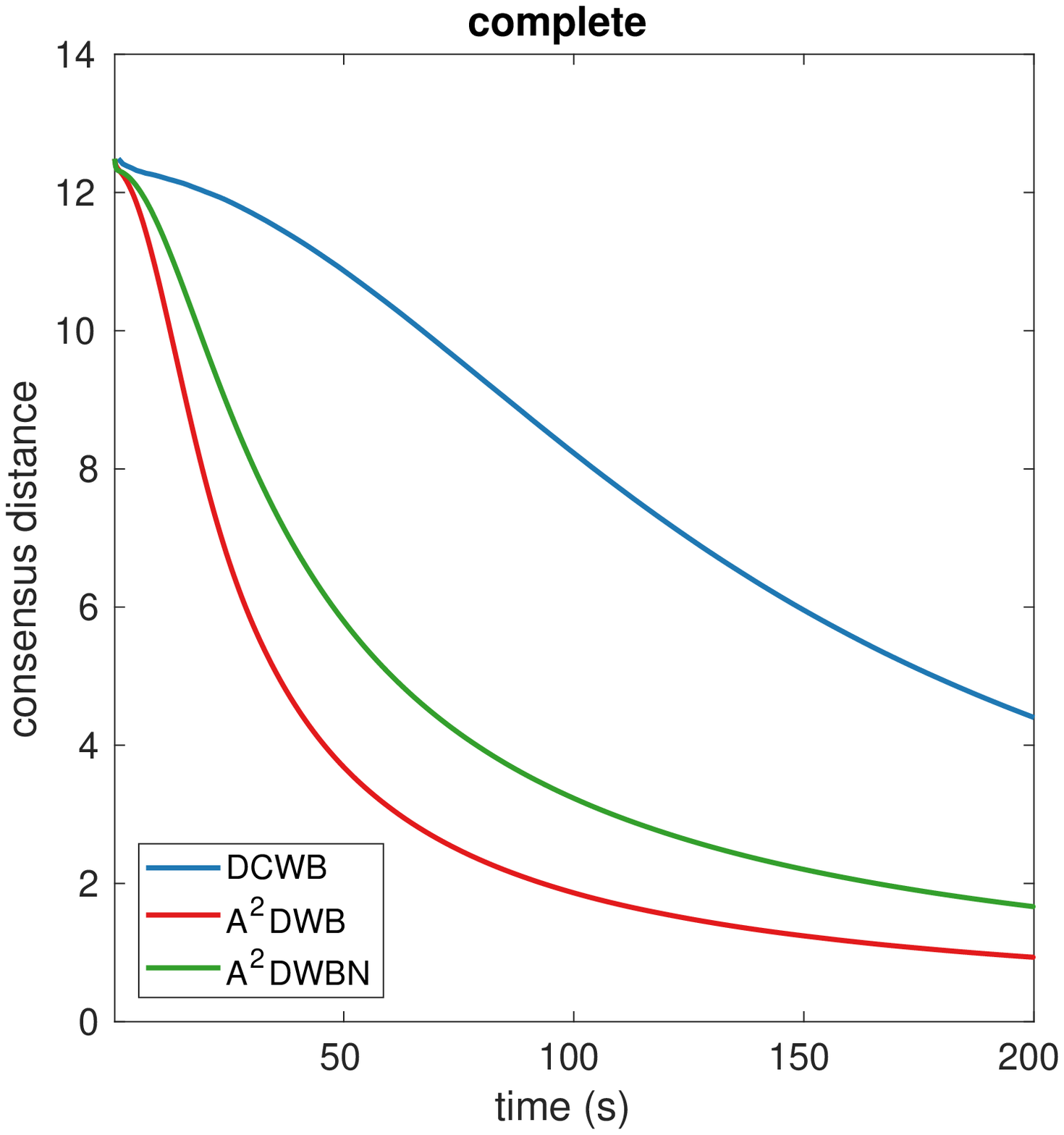}
	\end{subfigure}
	~
	\begin{subfigure}
		\centering
		\includegraphics[trim={0cm 0cm 0cm 0cm},clip,width=3.7cm,height= 3.2cm]{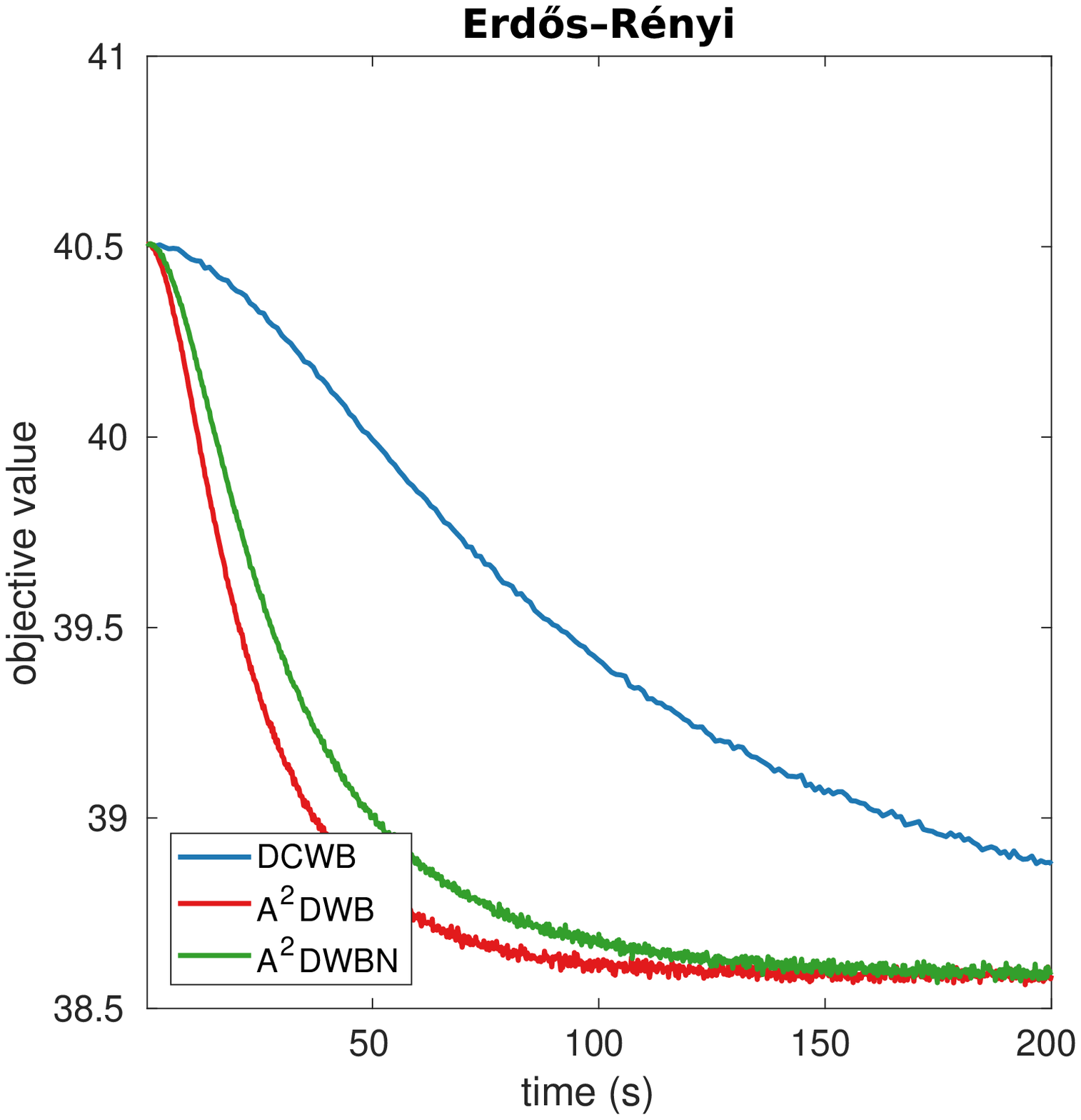}
	\end{subfigure}
	~
	\begin{subfigure}
		\centering
		\includegraphics[trim={0cm 0cm 0cm 0cm},clip,width=3.7cm,height= 3.2cm]{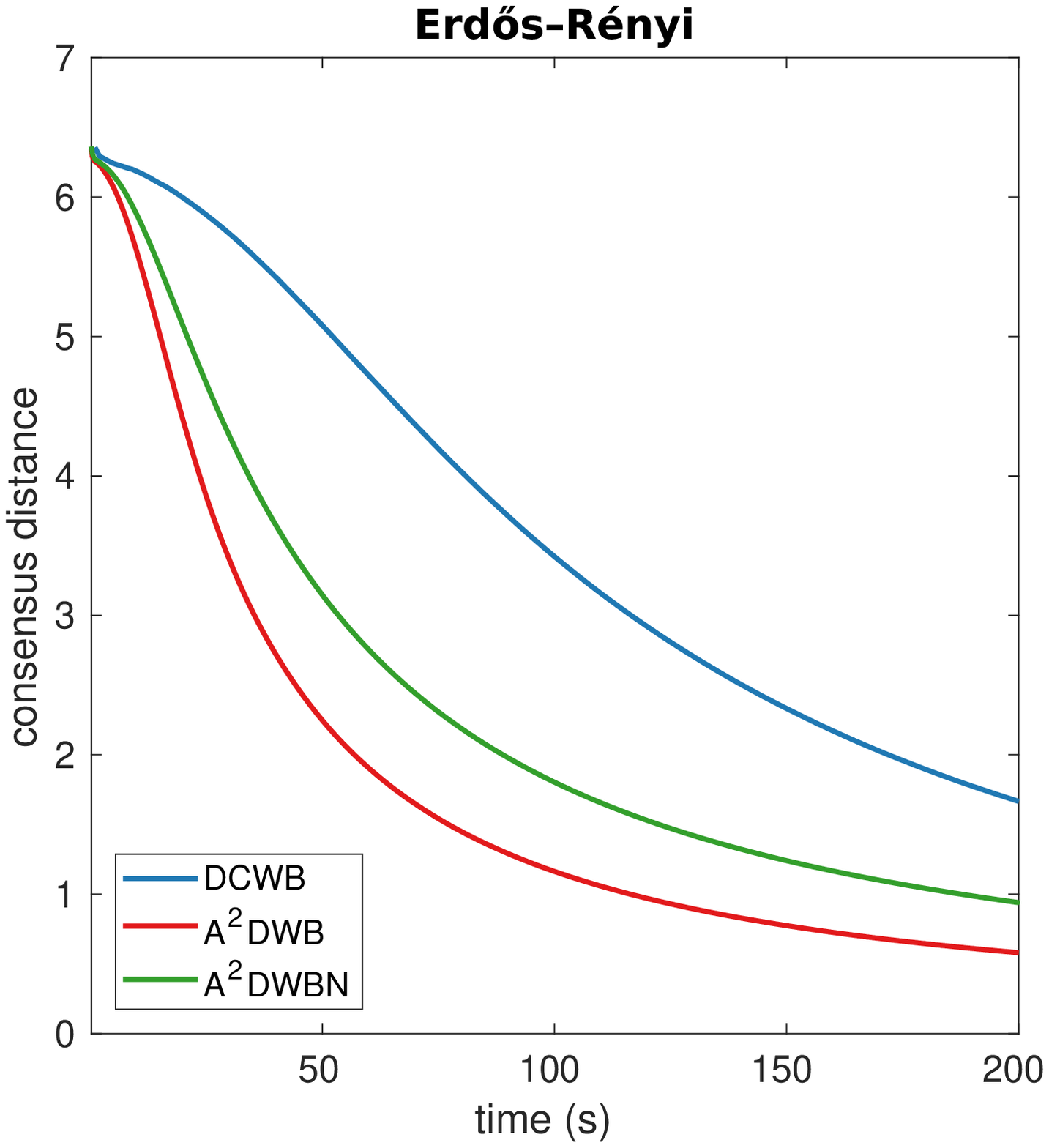}
	\end{subfigure}
	~
	\begin{subfigure}
		\centering
		\includegraphics[trim={0cm 0cm 0cm 0cm},clip,width=3.7cm,height= 3.2cm]{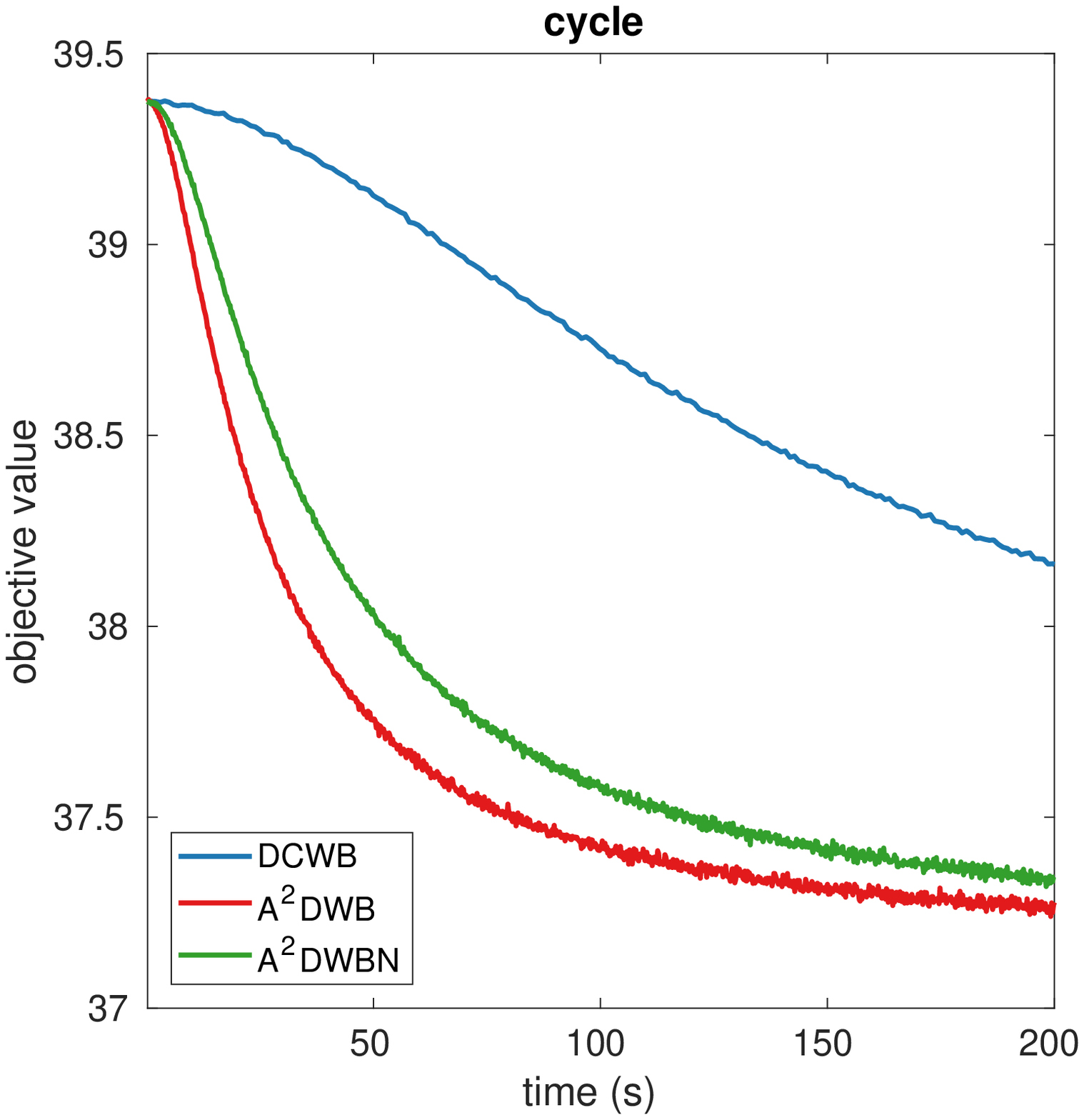}
	\end{subfigure}
	~
	\begin{subfigure}
		\centering
		\includegraphics[trim={0cm 0cm 0cm 0cm},clip,width=3.7cm,height= 3.2cm]{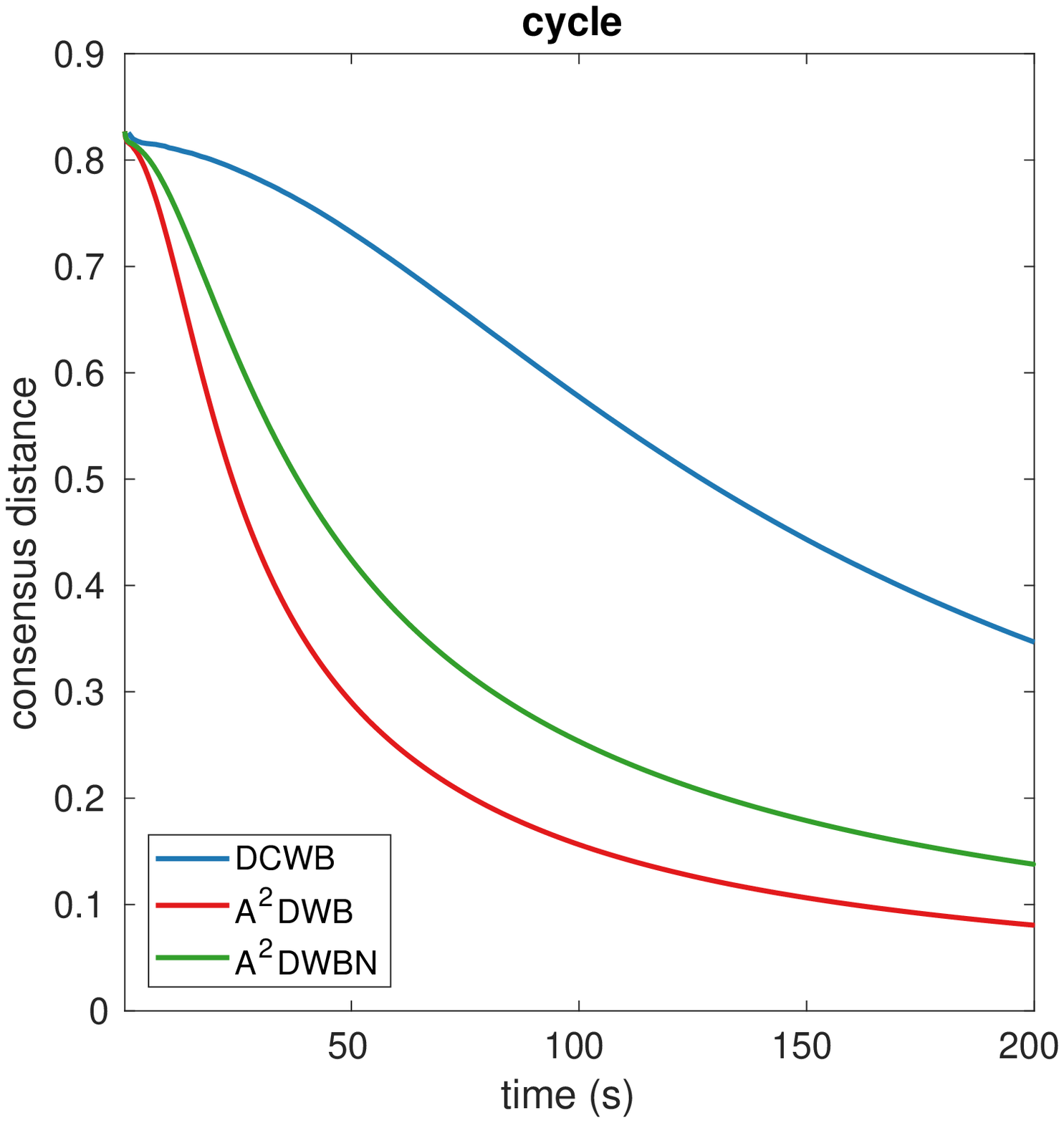}
	\end{subfigure}
	~
	\begin{subfigure}
		\centering
		\includegraphics[trim={0cm 0cm 0cm 0cm},clip,width=3.7cm,height= 3.2cm]{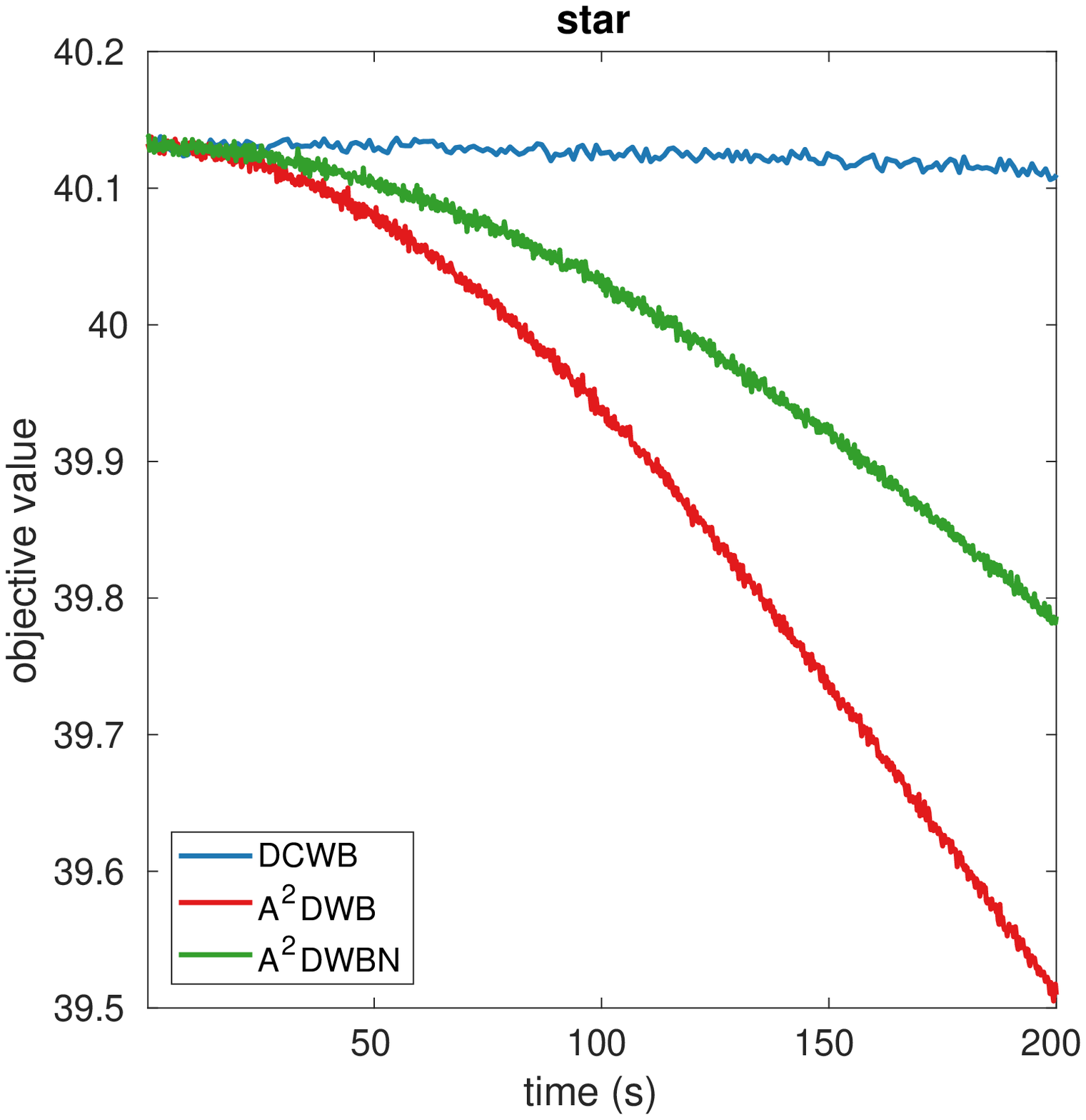}
	\end{subfigure}
	~
	\begin{subfigure}
		\centering
		\includegraphics[trim={0cm 0cm 0cm 0cm},clip,width=3.7cm,height= 3.2cm]{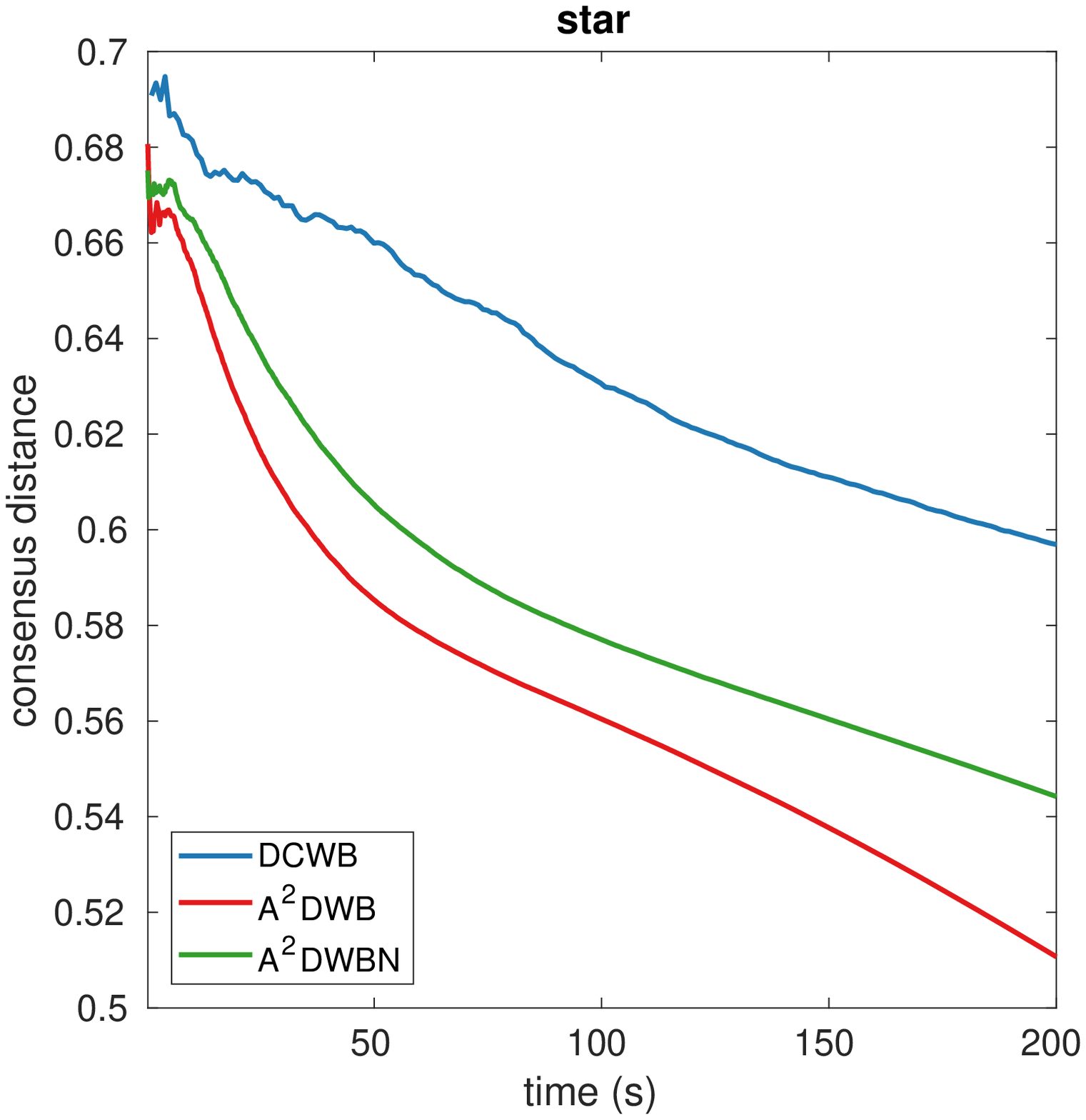}
	\end{subfigure}
	\caption{Simulation Experiment with Gaussian Distribution.From the top to the bottom, we sequentially report the results on complete, Erd\H{o}s-R\'{e}nyi, cycle and star networks.}
	\label{fig:exp2}
	\vskip -0.2in
\end{figure}

\subsection{MNIST Dataset}
In this experiment, we randomly choose 500 images of the same digit ($28 \times 28$ pixels)  from the MNIST data.
The pixel values of each image are normalized to add up to 1.
We assign each node one of the images and the objective is to jointly compute the Wasserstein Barycenter of the 500 samples present in the network. 
We run the experiment with digits 2,3,5, and 7 on all of the four topologies. 

The results of different digits on the same topology are similar.
Due to the limit of space, we only report part of the results in Figure \ref{fig:exp2}:
from the top to the bottom are the results of digit 2 on complete graph, digit 3 on Erd\H{o}s-R\'{e}nyi graph, digit 5 on cycle graph, and digit 7 on star graph, respectively.
Similar phenomenon as in experiment 1 has been observed in this task:
A$^2$DWB has better performance than the other two algorithms on both of the two comparison metrics, and the connectivity property of the underlying network topology effects the convergence of the algorithms.

\section{Conclusion}
In this paper, a practical asynchronous accelerated algorithm (A$^2$DWB) was proposed by applying a novel accelerated stochastic block coordinate descent algorithm (ASBCDS) to the dual of entropy regularized WBP.
Unlike its synchronous counterpart, A$^2$DWB needs no global synchronization, and thus substantially improve the time efficiency.
Theoretical analyses are provided for the proposed algorithms.
Empirical results validates the time efficiency of the A$^2$DWB compared to its synchronous counterpart.

\bibliographystyle{named}
\bibliography{ijcai19}

\end{document}


\section{Appendix of paper 3181 "An Asynchronous Decentralized Algorithm for Wasserstein Barycenter Problem"}
	\begin{theorem}
		Given a dual variable $\eta$ and its gradient $x = x^*(\sqrt{W}\eta)$ , the distance between $x$ and the optimal of the primal objective is bounded by $\|x - x^*\|^2 \le \frac{2}{\mu}(\varphi(\eta) - \varphi(\eta^*))$ and the consensus distance is bounded by
		$\|\sqrt{W} x\|^2 \le \frac{\lambda_{max}(W)}{\mu}(\varphi(\eta)-\varphi(\eta^*))$,
		where $x^*$ and $\eta^*$ denote the optimal solution to the primal and the dual problem,respectively.
	\end{theorem}
	\begin{proof}
		As $x=x^*(\sqrt{W}\eta)$, then we have 
		\begin{align}
		&\varphi(\eta) = \langle \sqrt{W}\eta,x\rangle - F(x),\label{3}\\
		&\varphi(\eta^*) = \langle \sqrt{W}\eta^*,x^*\rangle - F(x^*) = -F(x^*)\label{4},
		\end{align}
		where we use the fact that $\sqrt{W}x^* =0$.
		
		Subtracting \ref{4} from \ref{3}, we obtain
		\begin{equation}\label{5}
		\varphi(\eta) - \varphi(\eta^*) = \langle \sqrt{W}\eta,x-x^*\rangle + F(x^*) -F(x),
		\end{equation}
		where we use again $\sqrt{W}x^* =0$.
		
		Also, according to the strongly convexity of $F(x)$, we can write
		\begin{align}\label{6}
		\|x - x^*\|^2 &\le \frac{2}{\mu}(F(x^*) -F(x) + \langle \nabla F(x),x - x^*\rangle)
		\nonumber \\
		&= \le \frac{2}{\mu}(F(x^*) -F(x) + \langle\sqrt{W}\eta ,x - x^*\rangle),
		\end{align}
		where the equality follows from the fact that $\nabla F(x) = \sqrt{W}\eta$, since $x=x^*(\sqrt{W}\eta)$.
		
		Combining (\ref{5}) and (\ref{6}) leads to 
		\begin{equation*}
		\|x-x^*\|^2 \le \frac{2}{\mu}(\varphi(\eta) - \varphi(\eta^*))
		\end{equation*}
		
		Also by the $L=\frac{\lambda_{max}(W)}{\mu}$-smooth of $\varphi(\eta)$,
		we have 
		\begin{equation*}
		\|\varphi(\eta) -\varphi(\eta^*)\|^2 \le L(\varphi(\eta) -\varphi(\eta^*) -\langle \nabla \varphi(\eta^*),\eta - \eta^*\rangle)).
		\end{equation*}
		Substituting $\nabla \varphi(\eta) = \sqrt{W} x$ and $\nabla \varphi(\eta^*) =0$, we can conclude 
		$$\|\sqrt{W} x\|^2 \le \frac{\lambda_{max}(W)}{\mu}(\varphi(\eta)-\varphi(\eta^*)).$$

	\end{proof}
	
	In the following lemma, we summarize the properties of the sequence $\theta_k$ in ASBCDS.
	\begin{lemma}\label{lemma:theta}
	Assume that $\theta_1 =  1/m$ and $\theta_{k+1} = \frac{\sqrt{\theta^4_k+4 \theta_k^2}-\theta_k^2}{2}$ for $k \ge 1$, then $\theta_k$ satisfies $\frac{1}{k-1+2m}\le \theta_k \le \frac{2}{k-1+2m}$ and $\frac{1-\theta_{k+1}}{\theta_{k+1}^2} = \frac{1}{\theta_k^2}$.
	\end{lemma}
	\begin{proof}
		If we denote $h(a) = \frac{\sqrt{a^4+4a^2}-a^2}{2}$ as a function of $a$, then we have $\nabla h(a) > 0$ for $a\in (0,1)$.
		Thus, if $\frac{1}{k-1+2m}\le \theta_k \le \frac{2}{k-1+2m}$, then $h(\frac{1}{k-1+2m})\le \theta_{k+1} \le h(\frac{2}{k-1+2m})$.
		It can be verified that $h(\frac{1}{k-1+2m})\ge\frac{1}{k+2m}$ and $h(\frac{2}{k-1+2m})\le\frac{2}{k+2m}$.
		Since $\theta_1= 1/m$ satisfies the inequality, then we can conclude the first part of the lemma by induction.
		As for the second part of the conclusion, if we treat $\theta_{k+1}$ as the unknown variable and solve the equation, we can obtain $\theta_{k+1} = \frac{\sqrt{\theta^4_k+4 \theta_k^2}-\theta_k^2}{2}$.
	\end{proof}
	
	\begin{theorem}
	Under the assumption that $\varphi(\eta)$ is $L$-smooth, the stochastic gradient estimation is bounded as $\E \| \nabla \varphi(\omega_{j(k+1)})- \nabla \phi(\omega_{j(k+1)},\xi_{k+1})\|^2  \le \frac{mL\theta_{k+1}\epsilon}{8}$, and the delay $\tau \le m$, then for Algorithm ASBCDS, if step size $\gamma$ satisfies $3 L \gamma + 12 L \gamma(\frac{\tau^2+\tau }{m}+2\tau)^2 \le 1$, we have $\varphi(\eta_k) -\varphi(\eta^*) \le \epsilon$ after $K = \frac{\sqrt{2m^2 (\varphi(\eta_0) - \varphi(\lambda^*)+\|\zeta_0-\lambda^*\|^2/(2\gamma))}}{\sqrt{\epsilon}} =\OM(\frac{m\tau\sqrt{L}}{\sqrt{\epsilon}})$ iteration.
	Furthermore, if we assume $\E \| \nabla \varphi(\lambda)- \nabla \phi(\lambda,\xi)\|^2  \le\sigma^2$ and sample $M_k$ mini-batch of samples in the $k$-th iteration to satisfy the assumption on $\nabla \phi(\omega_{j(k+1)},\xi_{k+1})$, then the total stochastic oracle access is bounded by $K +1+ \frac{16\sigma^2 (K+2m)^2}{Lm\epsilon}=\OM(\frac{m\tau L}{\sqrt{\epsilon}}+\frac{m\tau^2\sigma^2}{\epsilon^2})$.
	\end{theorem}
	\begin{proof}

{\bf Step 1:}
	According to the update rule in ASBCDS, we have
	\begin{displaymath}
	\theta_{k+1} \zeta_k =\lambda_{k+1}-  (1-\theta_{k+1}) \eta_k ,
	\end{displaymath}
	and 
	\begin{align*}
	\eta_{k+1} &= \lambda_{k+1} + m \theta_{k+1}(\zeta_{k+1}-\zeta_k)
	\\
	&=(1-\theta_{k+1})\eta_{k} + m \theta_{k+1}\zeta_{k+1} -(m-1) \theta_{k+1}\zeta_k.
	\end{align*}
	Combing these together and eliminating $\zeta_k$, 
	we have 
	\begin{align*}
	m\theta_{k+1} \zeta_{k+1} = &\eta_{k+1} -(1-\theta_{k+1})\eta_k  +(m-1)\lambda_{k+1}
	\\& -(m-1)(1-\theta_{k+1})\eta_k.
	\end{align*}
	Then we have for $k \ge 1$
	\begin{displaymath}
	\frac{\lambda_{k+1} -(1-\theta_{k+1})\eta_k}{\theta_{k+1}} = \frac{\eta_k -m(1-\theta_k)\eta_{k-1}+(m-1)\lambda_{k}}{m\theta_k}.
	\end{displaymath}
	Thus, we have for each block $p \in \{1,\cdots,m\}$,
	\begin{align}\label{lamda-eta}
	\lambda_{k+1}^{[p]} =& \eta_k^{[p]} -\theta_{k+1}\eta_k^{[p]}  + \frac{\theta_{k+1}\eta_k^{[p]} }{m\theta_k}
	-\frac{\theta_{k+1}(1-\theta_k)\eta_{k-1}^{[p]} }{\theta_k} \nonumber\\
	&+\frac{(m-1)\theta_{k+1}\lambda_k^{[p]} }{m\theta_k}
	\nonumber\\
	=& \eta_k^{[p]} -\frac{\theta_{k+1}}{\theta_k}(\frac{1}{m} - \theta_k) (\eta_k^{[p]}  -\lambda_k^{[p]} )
	\nonumber\\
	&+\frac{\theta_{k+1}(1-\theta_k)}{\theta_k}(\lambda_{k}^{[p]}  -\eta_{k-1}^{[p]} ).
	\end{align}
	By setting $d_k =\frac{\theta_{k+1}(1-\theta_k)}{\theta_k}$ ,$e_k = \frac{\theta_{k+1}}{\theta_k}(\frac{1}{m} - \theta_k)$ and $b(h,k) = \prod_{i=h}^k d_k$, we have for $k  \ge jp(k+1) \ge 1$
	\begin{align*}
	\lambda_{k+1}^{[p]} =& \eta_k^{[p]} + e_k(\eta_k^{[p]}-\lambda_k^{[p]})+d_k(\lambda_k^{[p]}-\eta_{k-1}^{[p]})
	\\
	=& \lambda_k^{[p]} + (e_k+1)(\eta_k^{[p]}-\lambda_k^{[p]})+d_k e_{k-1}(\eta_{k-1}^{[p]}-\lambda_{k-1}^{[p]}) 
	\\
	&+d_kd_{k-1}(\lambda_{k-1}^{[p]}-\eta_{k-2}^{[p]})
	\\
	=& \lambda_k^{[p]} + (e_k+1)(\eta_k^{[p]}-\lambda_k^{[p]})
	\\
	&+\sum_{i=jp(k+1)}^{k-1}b(i+1,k)e_i (\eta_{i}^{[p]}-\lambda_{i}^{[p]}) 
	\\
	&+b(jp(k+1),k)(\lambda_{jp(k+1)}^{[p]}-\eta_{jp(k+1)-1}^{[p]})
	\end{align*}
	
	Making $l =k$ and summing the above inequality from $l = jp(k+1)$ to $k$, we have
	\begin{small}
	\begin{align*}
	\lambda_{k+1}^{[p]} =& \eta_{jp(k+1)}^{[p]} + \sum_{l=jp(k+1)}^k (e_l+1)(\eta_{l}^{[p]}-\lambda_l^{[p]})
	\\
	&+\sum_{l=jp(k+1)+1}^k\sum_{i=jp(k+1)}^{l-1}b(i+1,l)e_i (\eta_{i}^{[p]}-\lambda_{i}^{[p]}) 
	\\
	&+\sum_{l = jp(k+1)}^k b(jp(k+1),i)(\lambda_{jp(k+1)}^{[p]}-\eta_{jp(k+1)-1}^{[p]})
	\\
	=& \eta_{jp(k+1)}^{[p]} + \sum_{l=jp(k+1)}^k (e_l+1)(\eta_{l}^{[p]}-\lambda_l^{[p]})
	\\
	&+\sum_{i=jp(k+1)}^{k}(\sum_{l=i+1}^{k}b(i+1,l)e_i )(\eta_{i}^{[p]}-\lambda_{i}^{[p]}) 
	\\
	&+\sum_{l = jp(k+1)}^k b(jp(k+1),i)(\lambda_{jp(k+1)}^{[p]}-\eta_{jp(k+1)-1}^{[p]})
	\end{align*}
	\end{small}
	
	By summing over all the blocks, we have
	\begin{align}\label{eqn:lambda-omega}
	\lambda_{k+1} -\omega_{j(k+1)} = & \sum_{p=1}^{m} \sum_{h = jp(k+1)}^k(1+e_h+e_h\sum_{i=h+1}^k b(h,i))
	\nonumber\\
	&([\eta_h^{[p]}]-[\lambda_h^{[p]}]).
	\end{align} ,

	Then we have
	\begin{small}
	\begin{align}\label{eqn:omeganormp}
	&\|\lambda_{k+1} -\omega_{j(k+1)} \|^2
	\nonumber\\
	= &\sum_{p=1}^{m}  \|\sum_{h = jp(k+1)}^k(1+e_h+e_h\sum_{i=h+1}^k b(h,i))(\eta_h^{[p]}-\lambda_h^{[p]})\|^2
	\nonumber\\
	\le&\sum_{p=1}^{m} (\sum_{h = jp(k+1)}^k(1+e_h+e_h\sum_{i=h+1}^k b(h,i))) 
	\nonumber\\
	&\cdot\sum_{h = jp(k+1)}^k(1+e_h+e_h\sum_{i=h+1}^k b(h,i))\|\eta_h^{[p]}-\lambda_h^{[p]}\|^2
	\nonumber\\
	\le&\sum_{p=1}^{m} (\sum_{h = jp(k+1)}^k(1+e_h\sum_{i=1}^{k-h+1} 1)) 
	\nonumber\\
	&\cdot\sum_{h = jp(k+1)}^k(1+e_h\sum_{i=1}^{k-h+1} 1)\|\eta_h^{[p]}-\lambda_h^{[p]}\|^2
	\nonumber\\
	\le&\sum_{p=1}^{m} (\sum_{ii=1}^{k- jp(k+1)+1}(1+\frac{1}{m}\sum_{i=1}^{ii} 1)) 
	\nonumber\\
	&\cdot\sum_{ii = 1}^{k-jp(k+1)+1}(1+\frac{1}{m}\sum_{i=1}^{ii} 1)\|\eta_{k-ii+1}^{[p]}-\lambda_{k-ii+1}^{[p]}\|^2
	\nonumber\\
	\le&\sum_{p=1}^{m} (\sum_{ii=1}^{\min(k-\tau,\tau)}(1+\frac{1}{m}\sum_{i=1}^{ii} 1)) 
	\nonumber\\
	&\cdot\sum_{ii = 1}^{\tau}(1+\frac{1}{m}\sum_{i=1}^{ii} 1)\|\eta_{k-ii+1}^{[p]}-\lambda_{k-ii+1}^{[p]}\|^2
	\nonumber\\
	\le&\sum_{p=1}^{m} (\frac{\tau^2+\tau}{2m} + \tau) \sum_{i=1}^{\min(k-\tau,\tau)}(1+\frac{i}{m})
	\|\eta_{k-i+1}^{[p]}-\lambda_{k-i+1}^{[p]}\|^2,
	\end{align}
	\end{small}
	where 
	the first is due to the Cauchy-Schwartz inequality;
	in the second inequality, we use $b(h,i) \le 1$;
	in the third inequality, we change variable $ii = k-h+1$ and use $e_h \le \frac{1}{m}$;
	the forth inequality follows from the fact that $k+1 - jp(k+1) \le \tau$. 
	
	Dividing $\theta_{k+1}^2$ of both side of (\ref{eqn:omeganormp}) and summing from $k=0$ to $K$, we have 
	\begin{align}\label{eqn:omeganorm}
	&\sum_{k=0}^K \frac{1}{\theta_{k+1}^2} \|\lambda_{k+1} - \omega_{j(k+1)}\|^2
	\nonumber\\
	\le &\sum_{p=1}^{m} (\frac{\tau^2+\tau}{2m} + \tau)\sum_{k=0}^K \sum_{i=1}^{\min(\tau,k-\tau)}\frac{1+\frac{i}{m}}{\theta_{k+1}^2} 
	\|\eta_{k-i+1}^{[p]}-\lambda_{k-i+1}^{[p]}\|^2
	\nonumber\\
	\le& \sum_{p=1}^{m} (\frac{\tau^2+\tau}{2m} + \tau)
	\nonumber\\
	&\cdot\sum_{k=0}^K \sum_{i=1}^{\min(\tau,k-\tau)}\frac{16(1+\frac{i}{m})}{\theta_{k-i+1}^2} 
	\|\eta_{k-i+1}^{[p]}-\lambda_{k-i+1}^{[p]}\|^2
	\nonumber\\
	\le&4(\frac{\tau^2+\tau}{m} + 2\tau)^2\sum_{k=0}^K \frac{1}{\theta_{k+1}^2}\|\eta_{k+1}-\lambda_{k+1}\|^2
	\end{align}
	where the  second inequality follows from the fact that $\frac{1}{\theta_{k+1}^2} \le (k+2n)^2 \le 4(k-i+2n)^2 \le \frac{16}{\theta_{k+1-i}}$ according to Lemma \ref{lemma:theta} and $\tau \le n$,
	and the third is due to the fact that each $\frac{1}{\theta_{k+1}^2}\|\eta_{k+1}^{[p]}-\lambda_{k+1}^{[p]}\|^2$ appears at most $\tau$ times with coefficient from $1 + \frac{1}{m}$ to $1+\frac{\tau}{m}$.

{\bf Step2:}	
	From the smoothness of $\varphi$, we have
	\begin{align}\label{eta-oneiter}
	&\varphi(\eta_{k+1})
	\nonumber\\
	& \le \varphi(\lambda_{k+1}) +\langle \nabla \varphi(\lambda)^{[i_{k}]} ,\eta_{k+1}^{[i_{k}]}-\lambda_{k+1}^{[i_{k}]}\rangle + \frac{L}{2}\|\eta_{k+1}^{[i_{k}]}-\lambda_{k+1}^{[i_{k}]}\|^2
	\nonumber\\
	&\le\varphi(\lambda_{k+1}) +\frac{L}{2}\|\eta_{k+1}^{[i_{k}]}-\lambda_{k+1}^{[i_{k}]}\|^2 
	\nonumber\\
	&+ \langle \nabla \varphi(\lambda_{k+1})^{[i_{k}]} -\nabla \varphi(\omega_{j(k+1)})^{[i_{k}]},\eta_{k+1}^{[i_{k}]}-\lambda_{k+1}^{[i_{k}]}\rangle
	\nonumber\\
	& + \langle \nabla \varphi(\omega_{j(k+1)})^{[i_{k}]} - \nabla\phi(\omega_{j(k+1)},\xi_{k+1})^{[i_{k}]} ,\eta_{k+1}^{[i_{k}]}-\lambda_{k+1}^{[i_{k}]}\rangle
	\nonumber\\
	&+  \langle\nabla \phi(\omega_{j(k+1)},\xi_{k+1}) ^{[i_{k+1}]},\eta_{k+1}^{[i_{k+1}]}-\lambda_{k+1}^{[i_{k+1}]}\rangle
	\end{align}
	
	For the last three term of the r.h.s., we have
	\begin{align}\label{lh-1}
	&\langle \nabla  \varphi(\lambda_{k+1})^{[i_k]} -\nabla \varphi(\omega_{jl(k+1)})^{[i_k]} ,\eta_{k+1}^{[i_k]} -\lambda_{k+1}^{[i_k]} \rangle
	\nonumber\\
	&\le\frac{\gamma L_l^2}{2D_1}\|\lambda_{k+1}-\omega_{jl(k+1)}\|^2 + \frac{\gamma D_1 }{2}\|\frac{\lambda_{k+1} ^{[i_k]} -\eta_{k+1}^{[i_k]} }{\gamma}\|^2,
	\end{align}
	
	\begin{align}\label{lh-2}
	& \langle \nabla \varphi(\omega_{jl(k+1)})^{[i_k]}  - \nabla\phi(\omega_{j1(k+1)},\xi_{k+1})^{[i_k]}  ,\eta_{k+1}^{[i_k]} -\lambda_{k+1}^{[i_k]} \rangle
	\nonumber\\
	&\le \frac{\gamma }{2D_2}\| \nabla\varphi(\omega_{jl(k+1)})^{[i_k]}  - \nabla \phi(\omega_{j1(k+1)},\xi_{k+1})^{[i_k]} \|^2 
	\nonumber\\
	&+ \frac{\gamma D_2 }{2}\|\frac{\lambda_{k+1}^{[i_k]}  -\eta_{k+1}^{[i_k]}}{\gamma} \|^2,
	\end{align}
	
	and
	
	\begin{align}\label{lh-3}
	&\langle\nabla\phi(\omega_{j(k+1)},\xi_{k+1})^{[i_k]}  ,\eta_{k+1}^{[i_k]} -\lambda_{k+1}^{[i_k]} \rangle
	\nonumber\\
	&=\langle g_{k+1},\eta_{k+1}^{[i_k]}-\lambda_{k+1}^{[i_k]}\rangle
	\nonumber\\
	&=-\frac{\|\eta_{k+1}^{[i_k]}-\lambda_{k+1}^{[i_k]}\|^2}{\gamma}
	\end{align}
	
	By substituting (\ref{lh-1}),(\ref{lh-2}) and (\ref{lh-3}) into (\ref{eta-oneiter}),
	we have 
	\begin{align}
	&\varphi(\eta_{k+1})
	\nonumber\\
	&\le\varphi(\lambda_{k+1}) -\gamma(1-\frac{L\gamma}{2})\|\frac{\eta_{k+1}^{[i_k]}-\lambda_{k+1}^{[i_k]}}{\gamma}\|^2 
	\nonumber\\
	&+ \frac{\gamma }{2D_2}\| \nabla \varphi(\omega_{j(k+1)}) ^{[i_k]}- \nabla \phi(\omega_{j(k+1)},\xi_{k+1})^{[i_k]}\|^2 
	\nonumber\\
	&+ \frac{\gamma (D_1+D_2 )}{2}\|\frac{\lambda_{k+1}^{[i_k]} -\eta_{k+1}^{[i_k]}}{\gamma}\|^2
	\nonumber\\
	&+  \frac{\gamma L^2}{2D_1}\|\lambda_{k+1}-\omega_{j(k+1)}\|^2 
	\nonumber
	\end{align}
	
	Taking expectation on $i_k$, we have 
	\begin{align}\label{eqn:phi}
	&\E_{i_k}\varphi(\eta_{k+1})
	\nonumber\\
	&\le\varphi(\lambda_{k+1})+  \frac{\gamma L^2}{2D_1}\|\lambda_{k+1}-\omega_{j(k+1)}\|^2 
	\nonumber\\
	&+ \frac{\gamma }{2mD_2}\| \nabla \varphi(\omega_{j(k+1)})- \nabla \phi(\omega_{j(k+1)},\xi_{k+1})\|^2  
	\nonumber\\
	&-\gamma(1-\frac{L\gamma}{2}-\frac{D_1+D_2}{2})\E_{i_k}\|\frac{\eta_{k+1}^{[i_k]}-\lambda_{k+1}^{[i_k]}}{\gamma}\|^2 
	\end{align}
	
{\bf Step 3:}	
	We have 
	\begin{align}
	&\quad \frac{m^2}{2\gamma}\|\theta_{k+1}\zeta_{k+1} - \theta_{k+1}\eta^*\|^2 
	\nonumber\\
	&= \frac{m^2}{2\gamma}\|\theta_{k+1}\zeta_{k+1} - \theta_{k+1}\zeta_{k}+ \theta_{k+1}\zeta_{k} -\theta_{k+1}\eta^*\|^2 
	\nonumber\\
	&= \frac{m^2}{2\gamma}\|\theta_{k+1}\zeta_{k+1}^{[i_k]}-\theta_{k+1}\zeta_{k}^{[i_k]}\|^2 +\frac{m^2}{2\gamma}\|\theta_{k+1}\zeta_k -\theta_{k+1}\eta^* \|^2 
	\nonumber\\
	&\quad -m\langle  g_{k+1},\theta_{k+1}\zeta_k^{[i_k]} -\theta_{k+1}[\eta^*]^{[i_k]}\rangle.
	\nonumber
	\end{align}
	
	Taking expectation on $i_k$, we have
	\begin{align}\label{eqn:zeta}
	&\quad \frac{m^2}{2\gamma}\E_{i_k}\|\theta_{k+1}\zeta_{k+1} - \theta_{k+1}\eta^*\|^2 
	\nonumber\\
	&= \frac{1}{2\gamma}\E_{i_k}\|\eta_{k+1}^{[i_k]} - \lambda_{k+1}^{[i_k]}\|^2 +\frac{m^2}{2\gamma}\|\theta_{k+1}\zeta_k -\theta_{k+1}\eta^* \|^2 
	\nonumber\\
	&\quad \langle \nabla \varphi(\omega_{j(k+1)},\xi_{k+1}),\theta_{k+1}\zeta_k -\theta_{k+1}\eta^*\rangle.
	\end{align}

	If we further take expectation on $\xi_{k+1}$,then we have
	\begin{align}\label{eqn:prod}
	&\E -\langle \nabla \phi(\omega_{j(k+1)},\xi_{k+1}),\theta_{k+1}\zeta_k -\theta_{k+1}\eta^*\rangle
	\nonumber\\
	&=\E-\langle \nabla \phi(\omega_{j(k+1)},\xi_{k+1}),\lambda_{k+1}-(1-\theta_{k+1})\eta_{k}-
	\theta_{k+1}\eta^*\rangle
	\nonumber\\
	&=-\langle \nabla \varphi(\omega_{j(k+1)}),\lambda_{k+1}-(1-\theta_{k+1})\eta_{k}-
	\theta_{k+1}\eta^*\rangle
	\nonumber\\
	&=-\langle \nabla \varphi(\omega_{j(k+1)}),\omega_{jl(k+1)}-(1-\theta_{k+1}) \eta_{k}-
	\theta_{k+1}\eta^*\rangle 
	\nonumber\\
	&-\langle \nabla \varphi(\omega_{j(k+1)} , \lambda_{k+1} - \omega_{j(k+1)}  )
	\nonumber\\
	&\le (1-\theta_{k+1})\varphi(\eta_{k}) + \theta_{k+1} \varphi(\eta^*)- \varphi(\omega_{j(k+1)})
	\nonumber\\
	&-\langle \nabla \varphi(\omega_{j(k+1)} , \lambda_{k+1} -  \omega_{j(k+1)}  )
	\nonumber\\
	&\le (1-\theta_{k+1})\varphi(\eta_{k}) + \theta_{k+1} \varphi(\eta^*)- \varphi(\lambda_{k+1})
	\nonumber\\
	&+\langle \nabla \varphi(\lambda_{k+1}) -\nabla \varphi(\omega_{j(k+1)} ),\lambda_{k+1} -  \omega_{jl(k+1)}  ),
	\end{align}
	where the expectation is taken on $\xi_{k+1}$ and the first equality follows from the line (3) in ASBCDS,the second equality is due to the independence of $\nabla \phi(\omega_{jl(k+1)},\xi_{k+1})$ and  $\omega_{jl(k+1)}-(1-\theta_{k+1}) \eta_{k}-
	\theta_{k+1}\eta^*$.
	The first and second inequality just follows from the convexity of $\varphi$.
	
	{\bf Step 4:}
	By summing up (\ref{eqn:phi})(\ref{eqn:zeta}) and (\ref{eqn:prod}) and taking expectation on $\xi_{k+1}$, we have
	\begin{align}\label{eqn:varphi-oneiter}
	&\E_{i_{k+1}} \varphi(\eta_{k+1}) \le (1-\theta_{k+1}) \varphi(\eta_k) + \theta_{k+1} \varphi(\eta^*) 
	\nonumber\\
	&-\gamma(\frac{1}{2}-\frac{L\gamma}{2}-\frac{D_1+D_2}{2})\E_{\xi_{k+1}}[\E_{i_k}\|\frac{\eta_{k+1}-\lambda_{k+1}}{\gamma}\|^2 ]
	\nonumber\\
	&+ \frac{\gamma }{2mD_2}\E_{\xi_{k+1}} \| \nabla \varphi(\omega_{j(k+1)})- \nabla \phi(\omega_{j(k+1)},\xi_{k+1})\|^2 
	\nonumber\\
	&+  (\frac{\gamma L^2}{2D_1}+L)\|\lambda_{k+1}-\omega_{j(k+1)}\|^2 +\frac{m^2}{2\gamma}\|\theta_{k+1}\zeta_k -\theta_{k+1}\eta^* \|^2 
	\nonumber\\
	&-\E_{\xi_{k+1}}[\frac{m^2}{2\gamma}\E_{i_k}\|\theta_{k+1}\zeta_{k+1} - \theta_{k+1}\eta^*\|^2 ],
	\end{align}
	where we use the fact that $\|\eta_{k+1}-\lambda_{k+1} \|^2= \|\eta_{k+1}^{i_k}-\lambda_{k+1}^{i_k}\|^2$
	
	Dividing $\theta_{k+1}^2$ on both side of (\ref{eqn:varphi-oneiter}), we have
	\begin{align}\label{eqn:varphi-zeta}
	&\frac{\E \varphi(\eta_{k+1}) -\varphi(\eta^*)}{\theta_{k+1}^2} +\frac{m^2}{2\gamma}\E\|\zeta_{k+1} - \eta^*\|^2
	\nonumber\\
	\le &\frac{1-\theta_{k+1}}{\theta_{k+1}^2} (\varphi(\eta_k) - \varphi(\eta^*)) +\frac{m^2}{2\gamma}\|\zeta_k -\eta^* \|^2 
	\nonumber\\
	&-\frac{\gamma}{\theta_{k+1}^2}(\frac{1}{2}-\frac{L\gamma}{2}-\frac{D_1+D_2}{2})\E\|\frac{\eta_{k+1}-\lambda_{k+1}}{\gamma}\|^2 
	\nonumber\\
	&+ \frac{\gamma }{2mD_2\theta_{k+1}^2}\E_{\xi_{k+1}} \| \nabla \varphi(\omega_{j(k+1)})- \nabla \phi(\omega_{j(k+1)},\xi_{k+1})\|^2 
	\nonumber\\
	&+  \frac{1}{\theta_{k+1}^2}(\frac{\gamma L^2}{2D_1}+L)\|\lambda_{k+1}-\omega_{j(k+1)}\|^2  ,
	\end{align}
	
	Using the fact that $\frac{1-\theta_{k+1}}{\theta_{k+1}^2} = \frac{1}{\theta_k^2}$, summing from $k=0$ to $K$, and taking expectation on all the $\i_k$'s and $\xi_{k+1}$'s, we have
	\begin{align}
	&\frac{\E \varphi(\eta_{K+1}) -\varphi(\eta^*)}{\theta_{K+1}^2} +\frac{m^2}{2\gamma}\E\|\zeta_{K+1} - \eta^*\|^2
	\nonumber\\
	\le &\frac{1}{\theta_1^2} (\varphi(\eta_0) - \varphi(\eta^*)) +\frac{m^2}{2\gamma}\|\zeta_0 -\eta^* \|^2 
	\nonumber\\
	&-\sum_{k=0}^K\frac{\gamma}{\theta_{k+1}^2}(\frac{1}{2}-\frac{L\gamma}{2}-\frac{D_1+D_2}{2})\E\|\frac{\eta_{k+1}-\lambda_{k+1}}{\gamma}\|^2 
	\nonumber\\
	&+ \sum_{k=0}^K\frac{\gamma }{2mD_2\theta_{k+1}^2}\E \| \nabla \varphi(\omega_{j(k+1)})- \nabla \phi(\omega_{j(k+1)},\xi_{k+1})\|^2 
	\nonumber\\
	&+  \sum_{k=0}^K\frac{1}{\theta_{k+1}^2}(\frac{\gamma L^2}{2D_1}+L)\E\|\lambda_{k+1}-\omega_{j(k+1)}\|^2
	\nonumber\\
	\le &\frac{1}{\theta_1^2} (\varphi(\eta_0) - \varphi(\eta^*)) +\frac{m^2}{2\gamma}\|\zeta_0 -\eta^* \|^2 + \sum_{k=1}^K\frac{\epsilon }{16\theta_{k+1}}
	\nonumber\\
	&-(\frac{1}{2}-\frac{L\gamma}{2}-\frac{D_1+D_2}{2} - 4(\frac{\gamma^2 L^2}{2D_1}+L\gamma)
	(\frac{\tau^2+\tau}{m} + 2\tau)^2)
	\nonumber\\
	&\sum_{k=0}^K\frac{\gamma}{\theta_{k+1}^2}\E\|\frac{\eta_{k+1}-\lambda_{k+1}}{\gamma}\|^2,
	\end{align}
	where the second inequality follows from (\ref{eqn:omeganorm}) and $\E \| \nabla \varphi(\omega_{j(k+1)})- \nabla \phi(\omega_{j(k+1)},\xi_{k+1})\|^2  \le \frac{mD_2\theta_{k+1}\epsilon}{8\gamma}$.
	
	{\bf Step 5:}
	If we set $D_1 = D_2 = \gamma L$, and choose $\gamma$ satisfying $3 L \gamma + 12 L \gamma(\frac{\tau^2+\tau }{m}+2\tau)^2 \le 1$, then we have
	\begin{align}
	&\E\varphi(\eta_{K+1}) -\varphi(\eta^*) 
	\nonumber \\
	\le& \frac{\theta_{K+1}^2}{\theta_1^2}(\varphi(\eta_0) -\varphi(\eta^*) + \frac{\theta_1^2 m^2}{2\gamma}\|\zeta_0 -\eta^*\|^2) + \frac{\epsilon}{2},
	\end{align}
	where we use the fact that $\sum_{k=0}^K \frac{\theta_{K+1}^2}{\theta_{k+1}}\le \frac{4\sum_{k=0}^K (k+2m)}{(K+2m)^2}\le 8$.
	
	Then according to Lemma \ref{lemma:theta}, in order to get $\epsilon$ accuracy, we need to set $K = \frac{\sqrt{2m^2 (\varphi(\eta_0) - \varphi(\lambda^*)+\|\zeta_0-\lambda^*\|^2/(2\gamma))}}{\sqrt{\epsilon}} $.
	Since in the $k$-th iteration, we need to sample $M_k = \frac{8\sigma^2 \gamma}{m\epsilon L\gamma \theta_{k+1}} \le \frac{8 \sigma^2 (k+2m)}{mL\epsilon}$ mini-batch of samples,
	and the overall stochastic gradient access is $\sum_{k=0}^K \max(1,M_k) \le K+1 + \sum_{k=0}^{K} M_k \le K +1+ \frac{16\sigma^2 (K+2m)^2}{Lm\epsilon}$.
	\end{proof}
	
	\begin{theorem}[Equivalence between ASBCDS and PASBCDS]
		If we take the same $jp(k+1)$ and $\xi_{k+1}$ in each iteration of Algorithm ASBCDS and PASBCDS, then we have $\lambda_{k+1} =u_k + \theta_{k+1}^2 v_k $, $\zeta_{k+1} = u_{k+1}$ and $\eta_{k+1} = u_{k+1} + \theta_{k+1}^2 v_{k+1}$ for all $k = 0,\cdots, K$. 
	\end{theorem}
	We use $\omega_{jl(k)}^o$ and $g_k^o$ denote the $\omega_{jl(k)}$ and $g_k$ generated by ASBCD, and use $\omega_{jl(k)}^n$ and$g_k^n$to denote the $\omega_{jl(k)}$ generated by PASBCD.
	Besides, we also assume that if $\omega_{jl(k)}^o$=$\omega_{jl(k)}^n$, then we have $g_k^o =g_k^n$.
	Now, we need to prove that $u_k= \zeta_k$ and $\eta_k = u_k + \theta_k^2 v_k$ by induction.
	
	When $k = 0$, $\zeta_k = u_k  = 0$ and $\eta_k = u_k + \theta_k^2 v_k =0$.
	Since $jp(0) =0$ for all $p\in\{1,\cdots,m\}$,then  we have $\omega_{j(1)}^o =\lambda_1= \omega_{j(1)}^n=0 = u_0+ \theta_1^2 v_0$ for all $l\in \{1,\cdots,m\}$, which ensures $g_1^o = g_1^n$.
	Then we have $\zeta_1 = u_1$ and
	\begin{align*}
	\eta_1 &= \lambda_1 + m \theta_1(\zeta_1 -\zeta_0)  = u_0 + \theta_1^2 v_0  + m \theta_1(\zeta_1 -\zeta_0)
	\\
	&= u_0 + \theta_1^2 v_0  + m \theta_1(u_1 -u_0)
	\\
	&= u_1 + \theta_1^2(v_0 - \frac{1-m\theta_1}{\theta_1^2}(u_1-u_0))
	\\
	&= u_1 + \theta_1^2 v_1.
	\end{align*}
	
	When $k > 0$, suppose we have $u_k= \zeta_k$ and $\eta_k = u_k + \theta_k^2 v_k$ , then
	\begin{align}\label{ind:lambda}
	\lambda_{k+1} &= (1-\theta_{k+1}) \eta_k + \theta_{k+1} \zeta_k
	\nonumber\\
	&= (1-\theta_{k+1}) (\eta_k -\zeta_k) + \zeta_k
	\nonumber\\
	&= (1-\theta_{k+1})\theta_k^2 v_k + u_k = \theta_{k+1}^2v_k + u_k,
	\end{align}
	where in the last line, we use the fact that $(1-\theta_{k+1})\theta_k^2 = \theta_{k+1}^2$.
	
	If we have $\omega_{jl(k+1)}^o = \omega_{jl(k+1)}^n$ for all $l \in \{1,\cdots,m\}$, then we have $g_{k+1}^o = g_{k+1}^n$ and $\zeta_{k+1} = u_{k+1}$.
	For $\eta_{k+1}$, we have
	\begin{align}\label{ind:eta}
	\eta_{k+1} &= \lambda_{k+1} + m \theta_{k+1}(\zeta_{k+1} -\zeta_k)  
	\nonumber\\
	&= u_k + \theta_{k+1}^2 v_k  + m \theta_{k+1}(\zeta_{k+1} -\zeta_k)
	\nonumber\\
	&= u_k + \theta_{k+1}^2 v_k  + m \theta_{k+1}(u_{k+1} -u_k)
	\nonumber\\
	&= u_{k+1} + \theta_{k+1}^2(v_k - \frac{1-m\theta_{k+1}}{\theta_{k+1}^2}(u_{k+1}-u_k))
	\nonumber\\
	&= u_{k+1} + \theta_{k+1}^2 v_{k+1}.
	\end{align}
	
	Now we need to prove $\omega_{j(k+1)}^o = \omega_{j(k+1)}^n$ .
	We consider the following auxiliary Algorithm~\ref{alg:ASBCD-A} and Algorithm ~\ref{alg:PASBCD-A}.

	\begin{algorithm}[h]
		\caption{ASBCDS-Auxiliary}
		\label{alg:ASBCD-A}
		\textbf{Input}: $\hat \eta_{jp(k+1)} =\eta_{jp(k+1)}$,$\hat\zeta_{jp(k+1)}=\zeta_{jp(k+1)}$, and $\theta_{(jp(k+1))}$, for all $p \in \{1,\cdots,m\}$.
		
		\begin{algorithmic}[1] 
			\FOR{$p = \{1,\cdots,m\}$}
			\FOR{$i = jp(k+1),\cdots,k$}
			\STATE  $\hat\lambda_{i+1}^{[p]} = \theta_{i+1}\hat\zeta_i^{[p]}  + (1-\theta_{i+1})\hat\eta_i^{[p]} $.
			\STATE $\hat\zeta_{i+1}^{[p]} = \hat\zeta_i^{[p]} $ ,
			\STATE$\hat\eta_{i+1} ^{[p]} =\hat \lambda_{i+1} ^{[p]} + \theta_{i+1}(\hat\zeta_{i+1}^{[p]} -\hat\zeta_i^{[p]} )$.
			\ENDFOR
			\ENDFOR
		\end{algorithmic}
		\textbf{Output}: $\eta_{k+1}$
	\end{algorithm}
	
	From Algorithm ~\ref{alg:ASBCD-A}, according to (\ref{eqn:lambda-omega}), we can conclude
	$\hat \lambda_{k+1} = \omega_{jl(k+1)}^o$ since $\hat\eta_i^{[p]} = \hat \lambda_i^{[p]}$ for all $p\in \{1,\cdots,m\}$ and $i\in\{jp(k+1),\cdots,k\}$.
	Note that in order to distinguish the variables, we use a superscript $\hat\cdot$ to indicate the variables generated by the auxiliary algorithms.
	
	\begin{algorithm}[th]
		\caption{PASBCDS-Auxiliary}
		\label{alg:PASBCD-A}
		\textbf{Input}: $\hat u_{jp(k+1)} =u_{jp(k+1)}$, 
		$\hat v_{jp(k+1)} =v_{jp(k+1)}$, and $\theta_{jp(k+1)}$
		for  all $p \in \{1,\cdots,m\}$
		\begin{algorithmic}[1] 
			\FOR {$p = 1,\cdots,m$}
			\FOR{$i = jp(k+1),\cdots,k$}
			\STATE $\hat\omega_{i+1}^{[p]} = \hat u_{i}^{[p]} +\theta_{i+1}^2\hat v_{i}^{[p]} ,$
			\STATE $\hat u _{i+1}^{[p]} =\hat u_i^{[p]} ,\hat v_{i+1}^{[p]} = \hat v_i^{[p]} $
			\ENDFOR
			\ENDFOR
		\end{algorithmic}
		\textbf{Output}: $\hat\eta_{k+1} = \hat u_{k+1} + \theta_{k+1}^2 \hat v_{k+1}$.
	\end{algorithm}
	Then follow the same induction in (\ref{ind:lambda}) and (\ref{ind:eta}), we can conclude that
	$\hat\eta_{k}^{[p]} = \hat u_{k}^{[p]} + \theta_k^2 \hat v_k^{[p]}$, $\hat \zeta_k^{[p]} = \hat u_k^{[p]}$ and $\hat \lambda_{k+1}^{[p]} = \hat u_k^{[p]} + \theta_{k+1}^2\hat v_k^{[p]} $.
	Since $ \hat u_{j(k+1)}^{[p]} = u_{jp(k+1)}^{[p]}$, and $\hat v_{j(k+1)}^{[p]}  = v_{jp(k+1)}$, we can conclude that $\hat \lambda_{k+1} = \omega_{j(k+1)}^n $ according to the update rule of $\omega^n$, i.e. line 3 in PASBCDS, which indicate $\omega_{k}^o = \omega_{k}^n$ and verifies the equivalence between ASBCDS and PASBCDS.

	\bibliographystyle{named}
	\bibliography{ijcai19}